\documentclass[10pt]{article} 
\usepackage[accepted]{tmlr}

\usepackage{algorithm}
\usepackage{algorithmic}
\usepackage[algo2e,linesnumbered,ruled]{algorithm2e}

\usepackage{amsfonts}
\usepackage{amsmath}
\usepackage{amssymb}
\usepackage{amstext}
\usepackage{amsthm}
\usepackage{mathtools}
\usepackage{dirtytalk}
\usepackage{diagbox}
\usepackage{placeins}

\usepackage[utf8]{inputenc} 
\usepackage[T1]{fontenc}    
\usepackage{hyperref}       
\usepackage[capitalize]{cleveref}
\usepackage{url}            
\usepackage{booktabs}       
\usepackage{amsfonts}       
\usepackage{nicefrac}       
\usepackage{microtype}      
\usepackage{xcolor}         
\usepackage{enumitem}
\usepackage{tikz}
\usetikzlibrary{calc, fadings, decorations, shapes, positioning, arrows}

\usepackage{caption}
\usepackage{subcaption}
\usepackage{graphicx}
\usepackage{float}

\usepackage{thmtools}
\usepackage{thm-restate}
\usepackage{wrapfig}
\usepackage{pifont}
\usepackage{xspace}

\newcommand{\alg}{\text{DR-SIT}\xspace}
\newcommand{\dataset}{\mathsf{D}}
\newcommand{\expect}[2]{\mathbb{E} _{#1} \left [ #2 \right ]}
\newcommand{\feat}{m}

\newcommand{\nsr}{\textbf{\textsc{NSR}}}
\newcommand{\pr}[1]{ \mathbb{P} \left ( #1 \right )}

\let\oldding\ding
\renewcommand{\ding}[2][1]{\scalebox{#1}{\oldding{#2}}}

\theoremstyle{plain}
\newtheorem{theorem}{Theorem}[section]

\newtheorem{lemma}[theorem]{Lemma}

\theoremstyle{definition}
\newtheorem{definition}[theorem]{Definition}

\theoremstyle{remark}

\title{Double Machine Learning Based Structure Identification from
Temporal Data}

\author{\name Emmanouil Angelis\thanks{Equal contribution.} \email emmanouil.angelis@helmholtz-munich.de \\
      \addr Helmholtz AI, Helmholtz Center Munich\\
      Technincal University of Munich
      \ANDONE
      \name Francesco Quinzan\footnotemark[1] \email francesco.quinzan@eng.ox.ac.uk  \\
      \addr Department of Engineering Science \\ University of Oxford
      \ANDONE
      \name Ashkan Soleymani \email ashkanso@mit.edu \\
      \addr Department of Electrical Engineering and Computer Science \\ Massachusetts Institute of Technology
      \ANDONE
      \name Patrick Jaillet \email jaillet@mit.edu \\
      \addr Department of Electrical Engineering and Computer Science \\ Massachusetts Institute of Technology
      \ANDONE
      \name Stefan Bauer \email st.bauer@tum.de \\
      \addr Helmholtz AI, Helmholtz Center Munich\\
      Technincal University of Munich
}

\def\month{09}  
\def\year{2025} 
\def\openreview{\url{https://openreview.net/forum?id=4iHAoFVM2K}} 

\begin{document}

\maketitle

\begin{abstract}
Learning the causes of time-series data is a fundamental task in many applications, spanning from finance to earth sciences or bio-medical applications. Common approaches for this task are based on vector auto-regression, and they do not take into account unknown confounding between potential causes. However, in settings with many potential causes and noisy data, these approaches may be substantially biased. Furthermore, potential causes may be correlated in practical applications or even contain cycles. To address these challenges, we propose a new double machine learning based method for structure identification from temporal data (DR-SIT). We provide theoretical guarantees, showing that our method asymptotically recovers the true underlying causal structure. Our analysis extends to cases where the potential causes have cycles, and they may even be confounded. We further perform extensive experiments to showcase the superior performance of our method. Code: \url{https://github.com/sdi1100041/TMLR_submission_DR_SIT}

\end{abstract}

\section{Introduction}
One of the primary objectives when working with time series data is to uncover the causal structures between different variables over time. Learning these causal relations and their interactions is of critical importance in many disciplines, e.g., healthcare~\citep{anwar2014multi}, climate studies~\citep{stips2016causal, runge2019inferring}, epidemiology~\citep{hernan2000marginal, robins2000marginal}, finance~\citep{hiemstra1994testing}, ecosystems~\citep{sugihara2012detecting}, and many more. Interventional data is not often accessible in many of these applications. For instance, in healthcare scenarios, conducting trials on patients may raise ethical concerns, or in the realm of earth and climate studies, randomized controlled trials are not feasible.

In general, understanding the underlying causal graph using only observational data is a cumbersome task due to many reasons: i) observational data, as opposed to interventional data, capture correlation-type relations instead of cause-effect ones. ii) unobserved confounders introduce biases and deceive the algorithms to falsely infer causal relations instead of the true structure, e.g., the existence of a hidden common confounder iii)  the number of possible underlying structures grows super exponentially with the number of variables creating major statistical and computation barriers iv) the identifiability problem, since multiple causal models can result in the same observational distribution, thus making it impossible to uniquely determine the true structure. To overcome these problems and determine the underlying structure, additional assumptions are imposed. Typical assumptions include faithfulness, linearity of relations, or even noise-free settings, which limit the types of causal relationships that can be discovered~\citep{pearlj,peters2017elements,spirtes2000causation,glymour2019review}. Almost all of these challenges extend to the problem of identifying the underlying causal structure from observational time-series datasets.

Subsequently, in many instances, the emphasis is placed on particular target variables of interest and their causal features. Causal features of a target are defined as the set of variables that conditioned on them, the target variable is independent of the rest variables. Causal feature selection enables training models which are much simpler, more interpretable, and more robust~\citep{aliferis2010local,janzing2020feature}. However, learning the causal structures between variables and a specific target is still a demanding task, and many current approaches for causal feature selection face limitations by making unrealistic simplifying assumptions about the data-generating process or by lacking computational and/or statistical scalability~\citep{yu2021unified,yu2020causality}. These challenges become particularly amplified in the context of time series data, where the number of variables grows linearly with the length of the data trajectories. As a result, to mitigate these problems, additional assumptions, e.g., stationarity or no hidden confounders are included~\citep{moraffah2021causal,bussmann2021neural,runge2018causal} and/or weaker notions of causality\footnote{weaker than Pearl’s structural equation model~\citep{pearlj}.} such as Granger causality have been studied extensively 
~\citep{granger1988some,granger1969investigating,marinazzo2011nonlinear,tank2018neural,bussmann2021neural,khanna2019economy,runge2018causal,hasan2023survey}.

Existing causal feature selection from time-series data algorithms often assume some level of faithfulness or causal sufficiency~\citep{runge2019detecting,runge2018causal,runge2020discovering}. Oftentimes, they overlook the presence of unknown confounding factors among potential causes~\citep{moraffah2021causal}. Moreover, most cannot adapt to cyclic settings~\citep{entner2010causal}, which is relatively ubiquitous in many domains~\citep{bollen1989structural}. Furthermore, many algorithms employ the popular vector auto-regression framework to model time-dependence structures~\citep{bussmann2021neural,lu2016integrating,chen2009granger,weichwald2020causal,hyvarinen2010estimation}, which again is restrictive. To overcome these problems, we propose an efficient algorithm for doubly robust structure identification from temporal data.

\noindent\textbf{Our contribution.}
\begin{enumerate} [leftmargin=4mm,itemsep=-1ex,topsep=0pt]
    \item We provide an efficient and easy-to-implement doubly robust structure identification from temporal data algorithm (DR-SIT) with theoretical guarantees enjoying $\sqrt{n}$-consistency.
      \item We provide an extensive technical discussion relating Granger causality to Pearl’s framework for time series and show under which assumptions our approach can be used for feature selection or full causal discovery. As a consequence of this, ours is the first paper to propose a non-parametric Granger causality test that achieves the semi-parametric $\sqrt{n}$-rate. {\color{black} We remark, however, that the same semi-parametric rate was previously established for the related notion of Conditional Local Independence by \citet{Christgau2022CLI}.}
    \item {\color{black}We provide theoretical insights showing that our algorithm allows for general non-linear cyclic structures and hidden confounders among the covariates, while only requiring faithfulness within the parental subgraph of the target. In particular, faithfulness outside this local subgraph is not necessary in the context of local causal discovery, where the goal is to identify the direct causes of the target.}
    \item In extensive experiments we illustrate that our approach is significantly more robust, significantly faster, and more performative than state-of-the-art baselines. 
\end{enumerate}

\section{Related Work}
\label{app:related_work}

\noindent\textbf{Causal Structure Learning for Timeseries} 
A longstanding line of work intends to tailor the existing causal structure learning and Markov blanket discovery for i.i.d. data to the temporal setting. To name a few, \citet{entner2010causal} adapted the Fast Causal Inference algorithm~\citep{spirtes2000causation} to time-series data. While the approach shares the benefit of being able to deal with hidden confounders, it is not possible to account for cyclic structures. \citet{runge2019detecting} introduced PCMCI, as an adjusted version of PC~\citep{spirtes2000causation} with an additional false positive control phase which is able to recover time-lagged causal dependencies. PCMCI+ modified the approach further to additionally be able to find contemporaneous causal edges~\citep{runge2020discovering}. LPCMCI extends the scope by catering to the case of hidden confounders~\citep{gerhardus2020high}. Even though methods in this category are able to provide theoretical guarantees for learning the causal structure, DR-SIT has several advantages over them: i) For all of these methods, the faithfulness assumption is a key ingredient while DR-SIT does not need it. ii) These methods are based on conditional independence testing which is widely recognized to be a cumbersome statistical problem ~\citep{bergsma2004testing,kim2022local}. \citet{shah2020hardness} have established that no conditional independence (CI) test can effectively control the Type-I error for all CI settings. Moreover in practice, conducting many conditional independence tests from lengthy time-series is burdensome. iii) Even having access to a perfect conditional independence test oracle, severe computational challenges exist~\citep{chickering1996learning,chickering2004large}\footnote{Learning Bayesian Networks with conditional independence tests is NP-Hard.} (please refer to~\cref{app:comp} for a detailed comparison). 

Another line of work relies heavily on the vector auto-regression framework. VarLiNGAM~\citep{hyvarinen2010estimation} generalizes LiNGAM~\citep{shimizu2006linear} to time-series and similar to that it assumes a linear non-Gaussian acyclic model for the data. In the work of~\citet{huang2019causal}, a time-varying linear causal model is assumed, allowing for causal discovery even in the presence of hidden confounders. More recently, deep neural networks are used to train vector auto-regression. \citet{tank2018neural}  adapted neural networks (named cMLP and cLSTM) for Granger causality by imposing group-sparsity regularizers. In a similar fashion, \citet{khanna2019economy} used recurrent neural networks. In another work, \citet{bussmann2021neural} designed neural additive vector auto-regression (NAVAR), a neural network approach to model non-linearities. In contrast to previous works, they extract Granger-type causal relations by injecting the necessary sparsity directly into the architecture of the neural networks. This line of work is quite limited to ours as they consider confining structural assumptions over the underlying causal structural equations; details on these assumptions are discussed next to \ref{cond:1} in \cref{sec:framework}.

\noindent\textbf{Double Machine Learning} The use of double robustness in causality problems has a long history mainly concentrated on estimating average treatment effect~\citep{robins1994estimation, funk2011doubly, benkeser2017doubly,bang2005doubly,sloczynski2018general}. ~\citet{10.1111/ectj.12097} introduced the DML framework to achieve double robustness for structural parameters. Upon that, \citet{korth} defined an orthogonalized score to infer the direct causes of partially linear models. Their approach is fast and allows for the assumption of complicated underlying structures but unfortunately, it is limited to only linear direct causal effects. While this was later extended to the non-linear case~\citep{quinzan2023drcfs}, we propose a doubly robust approach for identifying causal structures from temporal data under the general assumptions discussed in \cref{sec:framework}.

{\color{black}
%
\noindent\textbf{Conditional Local Independence.} While our focus is on causal discovery for discrete time series, it is helpful to discuss the related notion of Conditional Local Independence (CLI). CLI is an asymmetric, continuous-time notion of (non-)influence: $Y$ is conditionally locally independent of $X$ (given other histories) when augmenting the filtration with the history of $X$ does not change the predictable intensity (compensator) of $Y$. Foundational work introduced local-independence graphs and an appropriate separation criterion to represent such statements for point processes and related continuous-time models \citep{Didelez2007CFMP,Didelez2008LocalIndependence}. Subsequent theory developed graphical representations that remain meaningful under marginalization (latents), via directed mixed graphs and $\mu$-separation, thereby characterizing the Markov equivalence classes identifiable from observations \citep{MogensenHansen2020}. Related learning procedures for partially observed stochastic dynamical systems were proposed in \citet{MogensenMalinskyHansen2018}. Most relevant to testing, \citet{Christgau2022CLI} provide a model-free (nonparametric) CLI testing framework for counting-process targets: they define the Local Covariance Measure (LCM), estimate it with cross-fitted orthogonalized moments, and derive an $(\mathrm{X})$-Local Covariance Test ((X)-LCT) with uniform control of size and power under modest estimation rates. It is important to distinguish, however, that CLI concerns instantaneous (intensity-level) influence in continuous time. By contrast, the Granger-causal target in this paper is a discrete-time predictive notion: whether past values of $X$ improve prediction of $Y$ beyond the past of $Y$ and other covariates. The two notions answer related questions, but they are not interchangeable.
%
%
}

\section{Framework}
\label{sec:framework}
\subsection{Problem Description}

\begin{figure*}[t]
  \centering
  \resizebox {0.9\textwidth} {!} {
    \begin{tikzpicture}[node distance=10mm and 15mm, main/.style = {draw, circle, minimum size=0.8cm, inner sep=1pt}, >={triangle 45}] 

    \node[main] (1)               {$\scriptstyle X_0^1$};
    \node[main] (2) [right =of 1] {$\scriptstyle X_1^1$}; 
    \node[main] (3) [right =of 2] {$\scriptstyle X_2^1$}; 
    \node[main] (4) [right =of 3] {$\scriptstyle X_3^1$}; 
    \node[main] (5) [right =of 4] {$\scriptstyle X_4^1$}; 
    \node[main] (6) [right =of 5] {$\scriptstyle X_5^1$}; 
    \node (0) [left  =of 1] {};   
    \node (7) [right =of 6] {};

    \node[main] (11) [below =of 1]  {$\scriptstyle Y_0$};
    \node[main] (12) [right =of 11] {$\scriptstyle Y_1$};
    \node[main] (13) [right =of 12] {$\scriptstyle Y_2$};
    \node[main] (14) [right =of 13] {$\scriptstyle Y_3$};
    \node[main] (15) [right =of 14] {$\scriptstyle Y_4$};
    \node[main] (16) [right =of 15] {$\scriptstyle Y_5$};
    \node       (10) [left  =of 11] {};
    \node       (17) [right =of 16] {};

    \node[main] (21) [below =of 11] {$\scriptstyle X_0^2$};
    \node[main] (22) [right =of 21] {$\scriptstyle X_1^2$};
    \node[main] (23) [right =of 22] {$\scriptstyle X_2^2$};
    \node[main] (24) [right =of 23] {$\scriptstyle X_3^2$};
    \node[main] (25) [right =of 24] {$\scriptstyle X_4^2$};
    \node[main] (26) [right =of 25] {$\scriptstyle X_5^2$};
    \node       (20) [left  =of 21] {};
    \node       (27) [right =of 26] {};

    \node[main] (31) [below =of 21, dotted] {$\scriptstyle U_0$};
    \node[main] (32) [right =of 31, dotted] {$\scriptstyle U_1$};
    \node[main] (33) [right =of 32, dotted] {$\scriptstyle U_2$};
    \node[main] (34) [right =of 33, dotted] {$\scriptstyle U_3$};
    \node[main] (35) [right =of 34, dotted] {$\scriptstyle U_4$};
    \node[main] (36) [right =of 35, dotted] {$\scriptstyle U_5$};
    \node       (30) [left  =of 31]         {};
    \node       (37) [right =of 36]         {};

    \node       (40) [right =of 7]          {};
    \node[main] (41) [right =of 40]         {$\scriptstyle X^1$};
    \node[main] (42) [below =of 41]         {$\scriptstyle Y$  };
    \node[main] (43) [below =of 42]         {$\scriptstyle X^2$};
    \node[main] (44) [below =of 43, dotted] {$\scriptstyle U$  };

    \draw[->] (11) to [bend right] (13);
    \draw[->] (12) to [bend right] (14);
    \draw[->] (13) to [bend right] (15);
    \draw[->] (14) to [bend right] (16);

    \draw[->] (11) -- (2);
    \draw[->] (12) -- (3);
    \draw[->] (13) -- (4);
    \draw[->] (14) -- (5);
    \draw[->] (15) -- (6);
    
    \draw[->] (0) -- (1);
    \draw[->] (1) -- (2); 
    \draw[->] (2) -- (3);
    \draw[->] (3) -- (4);
    \draw[->] (4) -- (5);
    \draw[->] (5) -- (6);
    \draw[->] (6) -- (7);

    \draw[<->] (1) to [bend left] (21); 
    \draw[<->] (2) to [bend left] (22); 
    \draw[<->] (3) to [bend left] (23); 
    \draw[<->] (4) to [bend left] (24); 
    \draw[<->] (5) to [bend left] (25); 
    \draw[<->] (6) to [bend left] (26); 
    
    \draw[->] (1) -- (12); 
    \draw[->] (2) -- (13); 
    \draw[->] (3) -- (14);
    \draw[->] (4) -- (15);
    \draw[->] (5) -- (16);

    \draw[->] (10) -- (11);
    \draw[->] (11) -- (12); 
    \draw[->] (12) -- (13);
    \draw[->] (13) -- (14);
    \draw[->] (14) -- (15);
    \draw[->] (15) -- (16);
    \draw[->] (16) -- (17);

    \draw[->] (20) -- (21);
    \draw[->] (21) -- (22); 
    \draw[->] (22) -- (23);
    \draw[->] (23) -- (24);
    \draw[->] (24) -- (25);
    \draw[->] (25) -- (26);
    \draw[->] (26) -- (27);

    \draw[->] (21) -- (13); 
    \draw[->] (22) -- (14);
    \draw[->] (23) -- (15);
    \draw[->] (24) -- (16);

    \draw[->, dotted] (30) -- (31);
    \draw[->, dotted] (31) -- (32); 
    \draw[->, dotted] (32) -- (33);
    \draw[->, dotted] (33) -- (34);
    \draw[->, dotted] (34) -- (35);
    \draw[->, dotted] (35) -- (36);
    \draw[->, dotted] (36) -- (37);

    \draw[->, dotted] (31) -- (21); 
    \draw[->, dotted] (32) -- (22);
    \draw[->, dotted] (33) -- (23);
    \draw[->, dotted] (34) -- (24);
    \draw[->, dotted] (35) -- (25);
    \draw[->, dotted] (36) -- (26);
    
    \draw[->, dotted] (31) to [bend left] (1); 
    \draw[->, dotted] (32) to [bend left] (2);
    \draw[->, dotted] (33) to [bend left] (3);
    \draw[->, dotted] (34) to [bend left] (4);
    \draw[->, dotted] (35) to [bend left] (5);
    \draw[->, dotted] (36) to [bend left] (6);

    \draw[->]         (41) to (42);
    \draw[->]         (43) to (42);
    \draw[<->] (41) to [bend left] (43);
    \draw[->, dotted] (44) to [bend left] (41);
    \draw[->, dotted] (44) to (43);
    
  \end{tikzpicture}%
  }
  \caption{
    Example of a full-time graph (left), and the corresponding summary graph (right). Note that time series $\boldsymbol{X}^1$ and $\boldsymbol{X}^2$ causally influence the outcome $\boldsymbol{Y}$ with different lags. {\color{black}Note also that our framework allows for auto correlative lags, as well as lagged causal effects from the target to any of the poential causes.} The time series $\boldsymbol{U}$ is unobserved, and it acts as a confounder for $\boldsymbol{X}^1$ and $\boldsymbol{X}^2$. Moreover, there is a cycle between the confoundings variables $\boldsymbol{X}^1$ and $\boldsymbol{X}^2$ of the outcome variable $\boldsymbol{Y}$.
}
\label{fig:graph}
\end{figure*}
{\color{black}We are given given i.i.d. realizations of a joint time series $(\boldsymbol{Y}, \boldsymbol{X})$ observed over time, where $\boldsymbol{Y} \coloneqq \{Y_t\}_{t \in \mathbb{Z}}$ is a univariate target time series and $\boldsymbol{X} \coloneqq \{X_t^1, \dots, X_t^m\}_{t \in \mathbb{Z}}$ is a multivariate time series of potential causes.}
We assume that the {\color{black}target} time series is specified by some of the potential causes by a deterministic function with posterior additive noise, and no instantaneous effects. We can formalize this model as
\begin{enumerate} [leftmargin=20mm,label={Axiom (\Alph*)},itemsep=0pt,topsep=0pt]
   \item \label{cond:1} $Y_T = f(\mathsf{pa}_T(\boldsymbol{Y}), T) + \varepsilon_T$ {\color{black}for all time steps $T\in \mathbb{Z}$},
\end{enumerate}
with $\varepsilon_T$ exogenous independent noise and $\mathsf{pa}_T(\boldsymbol{Y}) \subseteq \{X_t^1, \dots, X_t^m\}_{t \in \mathbb{Z}}$ is a subset of random variables of the multivariate time series $\boldsymbol{X}$. Note that the independence of the additive noise is important for identifiability. In fact, if there are dependencies between the noise and the history, then one might run into identifiability problems. We refer the reader to Appendix \ref{sec:counterexample} for a counterexample. We are interested in identifying the time series $\boldsymbol{X}^i$ that directly affects the outcome $Y$. That is, we wish to identify time series $\boldsymbol{X}^i$ such that it holds $X^i_{t} \subseteq \mathsf{pa}_T(\mathbf{Y})$ for some time steps $t, T$. 

We further use the following assumptions:
\begin{enumerate}[leftmargin=20mm,label={Axiom (\Alph*)},itemsep=-1ex,topsep=0pt]\setcounter{enumi}{1}
\item \label{cond:2} there are no causal effects backward in time{\color{black}. Specifically, $X_t^i\notin \mathsf{pa}_T(\mathbf{Y})$, for all $i= 1, \dots, m$ and for all time steps $t, T\in \mathbb{Z}$ with $t>T$};
\item \label{cond:3} there are no instantaneous causal effect between $\boldsymbol{Y}$ and any of the potential causes $\boldsymbol{X}^i${\color{black}, i.e., $X_T^i\notin \mathsf{pa}_T(\mathbf{Y})$, for all $i= 1, \dots, m$ and for all time steps $T\in \mathbb{Z}$}.
\end{enumerate}
{\color{black}Note that according to \ref{cond:3}, instantaneous effects are allowed between potential causes $\boldsymbol{X}^i$ and $\boldsymbol{X}^j$, as illustrated, for example, in Fig. \ref{fig:graph}.} Both \ref{cond:2} and \ref{cond:3} appear in previous related work (see, e.g., \citet{DBLP:conf/nips/PetersJS13,DBLP:conf/icml/MastakouriSJ21,DBLP:conf/clear2/LoweMSW22}). Note that these axioms allow for cycles and hidden common confounders between the potential causes. \ref{cond:2} is a natural assumption as a system is called causal when the output of the system only depends on the past, not the future~\citep{peters2017elements,pearlj}. \ref{cond:3} poses additional restrictions on the class of models that we consider, since instantaneous effects may be relevant in some cases and applications \citep{DBLP:journals/corr/abs-2206-06169}\footnote{An example of cases where time series exhibit instantaneous causal effects is given by dynamical systems~\citep{mogensen2018causal,rubenstein2016deterministic}. In dynamical systems, a variable may instantaneously affect another variable of the model. Instantaneous effects have been studied in previous related work~\citep{DBLP:journals/corr/abs-2210-14706} but due to identifiability issues, they rely on stronger assumptions such as faithfulness.}. However, without this assumption, it is impossible to learn causes \emph{from observational data}. The necessity of \ref{cond:3} for causal discovery is a well-known fact \citep{DBLP:conf/nips/PetersJS13}\footnote{In general, \citet{DBLP:conf/nips/PetersJS13} show that causal discovery is impossible with instantaneous effects. Please refer to~\cref{app:instantaneous} for an example of this non-identifiability. However, \citet{DBLP:conf/nips/PetersJS13} also provide a special case in which the causal structure is identifiable with instantaneous effects. This special case occurs when the random variables of the model are jointly Gaussian, and the instantaneous effects are linear. Our framework extends to this special case.}.

\subsection{Generality over Previous Work}
Our framework retains some degree of generality over previous related work. In fact, \ref{cond:1}-\ref{cond:3} allow for hidden confounding and cycles between the potential causes (see Figure \ref{fig:graph}), providing a more general framework than the full autoregressive model studied, e.g., by \citet{DBLP:journals/jmlr/HyvarinenZSH10,DBLP:conf/nips/PetersJS13,DBLP:conf/clear2/LoweMSW22,DBLP:journals/corr/abs-2001-01885}. Furthermore, in contrast to several previous works \citep{khemakhem2020variational,gresele2021independent,lachapelle2022disentanglement,lippe2022citris,yao2021learning}, we do not assume that
the variables $Y_T, X_T^1, \dots, X_T^m$ are independent conditioned on the observed variables at previous time steps. Importantly, we also do not assume causal faithfulness in the full time graph, or weaker notions such as causal minimality.\footnote{Recall that a distribution is faithful to a causal diagram if no conditional independence relations are present, other than the ones entailed by the Markov property.\label{fn:faithfulness}} This is a major improvement over some previous works, e.g., \citet{DBLP:conf/icml/MastakouriSJ21,DBLP:journals/corr/abs-2210-14706}, since there is no reason to assume that faithfulness or causal minimality hold in practice. Theorem 1 by \citet{DBLP:journals/corr/abs-2210-14706} provides an identifiability result for a model with history-dependent noise and instantaneous effects. This result, however, requires causal minimality.

Furthermore, as discussed in ~\cref{app:related_work},  vector auto-regression methods enforce heavy structural assumptions on the underlying causal structural equations. NAVAR~\citep{bussmann2021neural} assumes that each variable is influenced by its causal parents exclusively in an additive way and higher-order interactions among them are precluded. In mathematical terms, they follow $Y_T = \beta + \sum_{X \in \mathsf{pa}_T(\boldsymbol{Y})}^{N} f_{X} (X_{t - \kappa:t-1}) + \varepsilon_t$, where $\beta$ is a bias term, $\kappa$ is the time lag, and $\varepsilon_t$ is an independent noise variable. 
{\color{black}VarLiNGAM \citep{hyvarinen2010estimation} imposes considerably stricter constraints on the structural equations' functional form {\color{black} than our framework}, restricting $f_X$ to be a linear transformation of $X_{t - \kappa:t-1}$. As a result, it cannot capture even simple nonlinear relationships - e.g., $Y_t = X_{t - 1}^{1} \times X_{t - 1}^{2}$ or $Y_t = \log (X_{t - 1}^{1} + X_{t - 1}^{2})$ - which are well within the scope of \ref{cond:1}. In contrast, our approach flexibly models arbitrary interactions between covariates and the timestep $T$, marking a significant advancement over prior methods.}

\subsection{Causal Structure}
We are interested in direct causal effects, which are defined by distribution changes due to interventions on the DGP. An intervention amounts to actively manipulating the generative process of a potential cause $\boldsymbol{X}^i$ at some time step $t$, without altering the remaining components of the DGP. Then, a time series $\boldsymbol{X}^i$ has a direct effect on $\boldsymbol{Y}$ if performing an \emph{intervention} on some temporal variable $X^i_t$ will alter the distribution of $Y_T$, for some time steps $t$, $T$.

{\color{black}
We consider interventions by which a random variable $X^j_t$ is set to a constant $X^j_t \gets x$. We denote with $Y_{T} \mid do(X^i_t = x)$ the outcome time series $\boldsymbol{Y}$ at time step $T$, after performing an intervention as described above. We can likewise perform multiple joint interventions, by setting a group of random variables $\boldsymbol I $ at different time steps, to pre-determined constants specified by an array $\boldsymbol i$. We use the symbol $Y_T \mid do(\boldsymbol I = \boldsymbol i)$ to denote the resulting post-interventional outcome, and we denote with $\pr{Y_T = y \mid do(\boldsymbol I = \boldsymbol i)}$ the probability of the event $\{Y_T \mid do(\boldsymbol I = \boldsymbol i) = y\}$.

Using this notation, a time series $\boldsymbol{X}^i$ has a direct causal effect on the outcome $\boldsymbol{Y}$, if performing different interventions on the variables $\boldsymbol{X}^i$, while keeping the remaining variables fixed, will alter the probability distribution of the outcome $\boldsymbol{Y}$. Formally, define the sets of random variables ${\color{black}\boldsymbol{I}_{<T}}\coloneqq \{X_{t}^1, \dots, X_{t}^n , Y_t\}_{ t < T}$, which consists of all the information before time step $T$. Similarly, define the random variable ${\color{black}\boldsymbol{X}^i_{<T}} \coloneqq \{X^i_t\}_{t < T}$, consisting of all the information of time series $\boldsymbol{X}^i$ before time step $T$. Define the variable ${\color{black}\boldsymbol{I}_{<T}^{\setminus i}}\coloneqq {\color{black}\boldsymbol{I}_{<T}} \setminus {\color{black}\boldsymbol{X}^i_{<T}}$, which consists of all the variables in ${\color{black}\boldsymbol{I}_{<T}}$ except for ${\color{black}\boldsymbol{X}^i_{<T}}$. Then, a time series $\boldsymbol{X}^i$ has a direct effect on the outcome $ \boldsymbol{Y}$ if it holds
\begin{equation}
\label{eq:direct_effect}   
\pr {Y_T = y \mid do \left ({\color{black}\boldsymbol{X}^i_{<T}} = \boldsymbol{x}' , {\color{black}\boldsymbol{I}_{<T}^{\setminus i}} =  \boldsymbol{i} \right )} \neq \pr {Y_T = y \mid do \left ({\color{black}\boldsymbol{X}^i_{<T}} = \boldsymbol{x}'' , {\color{black}\boldsymbol{I}_{<T}^{\setminus i}} =  \boldsymbol{i} \right )}
\end{equation}
We say that a time series $\boldsymbol{X}^i$ causes $\boldsymbol{Y}$, if there is a direct effect between $\boldsymbol{X}^i$ and $\boldsymbol{Y}$ as in Eq. \ref{eq:direct_effect}, for any time step $T$.
}

Following, e.g., \citet{DBLP:conf/icml/MastakouriSJ21}, we define the \emph{full time} graph $\mathcal{G}$ as a directed graph whose edges represent all direct causal effects among the variables at all time steps. Given the outcome $Y_T$ at a given time step, we refer to the parent nodes in the full time graph as its \emph{causal parents}. We further define the \emph{summary} graph whose nodes are $\boldsymbol{X}^i$ and $\boldsymbol{Y}$, and with directed edges representing causal effects between the time series. We refer the reader to Figure \ref{fig:graph} for a visualization of these graphs. Note that the causes of $\boldsymbol{Y}$ correspond to the parent nodes of $\boldsymbol{Y}$ in the summary graph. {\color{black}In this work, we always assume that the Markov property holds (see, e.g., \citet{peters2017elements}).\footnote{{\color{black}Recall that the distribution of the DGP fulfills the Markov property if each variable in the graph $\mathcal{G}$ is conditionally independent of its non-non-descendants, given its causal parents.}}}
\subsection{Granger Causality}
\label{sec:granger}
Granger causality \citep{granger1988some} is one of the most commonly used approaches to infer causal relations from observational time-series data. Its central assumption is that \say{cause-effect relationships cannot work against time}. Informally, if the prediction of the future of a target time-series $\boldsymbol{Y}$ can be improved by knowing past elements of another time-series $\boldsymbol{X}^i$, then $\boldsymbol{X}^i$ \say{Granger causes} $\boldsymbol{Y}$. Formally, we say that $\boldsymbol{X}^i$ Granger causes $\boldsymbol{Y}$ if it holds
\begin{equation}
\label{eq:def_granger}   
\pr {Y_T=y \mid {\color{black}\boldsymbol{X}^i_{<T}} = \boldsymbol{x} , {\color{black}\boldsymbol{I}_{<T}^{\setminus i}} = \boldsymbol i } \neq \pr { Y_T =y\mid  {\color{black}\boldsymbol{I}_{<T}^{\setminus i}} = \boldsymbol i } \footnote{{\color{black}This formulation assumes a discrete target variable $Y_T$ for notational simplicity. The definition extends naturally to continuous random variables by comparing conditional distributions or using conditional densities where they exist. We refer the reader, e.g., to \citet{shojaie2022granger} for a comprehensive discussion.}}
\end{equation}
for a non-zero probability event $\{{\color{black}\boldsymbol{X}^i_{<T}} = \boldsymbol{x} , {\color{black}\boldsymbol{I}_{<T}^{\setminus i}} = \boldsymbol i\}$, where ${\color{black}\boldsymbol{I}_{<T}}$ stands for the set $\{{\color{black}\boldsymbol{X}^1_{<T}}, {\color{black}\boldsymbol{X}^2_{<T}}, \dots, {\color{black}\boldsymbol{X}^m_{<T}}\}$ and ${\color{black}\boldsymbol{I}_{<T}^{\setminus i}}$ represents the set ${\color{black}\boldsymbol{I}_{<T}} \setminus \{{\color{black}\boldsymbol{X}^i_{<T}}\}$. 

Granger causality is commonly used to identify causes. Assuming stationary, multivariate Granger causality analysis usually fits a vector autoregressive model to a given dataset. This model can be then used to determine the causes of a target $\boldsymbol{Y}$. However, it is important to note that Granger causality does not imply true causality in general. This limitation was acknowledged by Granger himself in \citet{granger1988some}.

\section{Methodology}
\label{sec:method}
\subsection{Double Machine Learning (DML)}
DML is a general framework for parameter estimation, which uses debiasing techniques to achieve $\sqrt{n}$-consistency (see, e.g., \citet{10.1093/biomet/asaa054,DBLP:conf/icml/ChernozhukovNQS22}). In DML, we consider the problem of estimating a parameter $\theta_0$ as a solution of an equation of the form $\expect{}{\mathcal{L}(\theta_0, \boldsymbol \eta_0)} =0$. The score function $\mathcal{L}$ depends on two terms, the true parameter $\theta_0$ that we wish to estimate, and a nuisance parameter $\boldsymbol \eta_0$. We do not directly care about the correctness of our estimate of $\boldsymbol \eta_0$, as long as we get a good estimator of $\theta_0$. The nuisance parameter $\eta_0$ may induce an unwanted bias in the estimation process, resulting in slow convergence. To overcome this problem, we use score functions that fulfill the Mixed Bias Property (MBP) \citep{10.1093/biomet/asaa054}, and learn the desired parameter $\theta_0$ using DML.

The MBP is a property that ensures that small changes of the nuisance parameter do not significantly affect the score function computed around the true parameters $(\theta_0, \boldsymbol \eta_0)$ (see Definition 1 by \citet{10.1093/biomet/asaa054}). {\color{black}In} this work, we construct a score with the MBP{\color{black}, following} \citet{DBLP:conf/icml/ChernozhukovNQS22,10.1093/biomet/asaa054}. For a fixed time step $T$, let $\boldsymbol{X}$ be a random variable in the set of random variables $\boldsymbol{V}$, $g$ any real-valued function of $\boldsymbol{X}$ such that $\expect{}{g^2(\boldsymbol{X})} < \infty$. We consider parameters of the form $\theta_0 \coloneqq \expect{}{m(\boldsymbol{V}; g)}$, where $m(\boldsymbol{V}; g)$ is a linear moment functional in $g$. The celebrated Riesz Representation Theorem ensures that, under certain conditions, there exists a function $\alpha_0$ of $\boldsymbol{X}$ such that $\expect{}{m(\boldsymbol{V}; g)} = \expect{}{\alpha_0(\boldsymbol{X}) g(\boldsymbol{X})}$. The function $\alpha_0$ is called the \emph{Riesz Representer} (RR). \citet{automatic_debiased} shows that the Riesz representer can be estimated from samples. Using the RR, we can derive a score function for the parameter $\theta_0 $ with $g_0(\boldsymbol{X}) = \expect{}{Y\mid \boldsymbol{X}}$ that fulfills the MBP. This function is defined as
%
\begin{equation}
\label{eq:debiased_score}
\varphi( \theta, \boldsymbol \eta) \coloneqq m(\boldsymbol{V}; g ) + \alpha(\boldsymbol{X})\cdot(Y - g(\boldsymbol{X})) - \theta.
\end{equation}
Here, $\boldsymbol \eta \coloneqq (\alpha, g)$ is a nuisance parameter consisting of a pair of square-integrable functions. As shown by \citet{DBLP:conf/icml/ChernozhukovNQS22}, the score function Eq. \eqref{eq:debiased_score} yields $\expect{}{\varphi(\theta_0, \boldsymbol \eta)}
= - \expect{}{(\alpha(\boldsymbol{X})- \alpha_0(\boldsymbol{X}))(g(\boldsymbol{X})- g_0(\boldsymbol{X}))}$, which gives the MBP as in Definition 1 of \citet{10.1093/biomet/asaa054}.

For score functions that fulfill the MBP, we can use DML to partly remove the bias induced by the nuisance parameter. The DML is defined as follows.\footnote{We remark that DML requires a weaker assumption on the score function than the MBP, namely the Neyman Orthogonality Condition {\color{black} \citep{Neyman1965,10.1111/ectj.12097}}. For simplicity, we give a description of DML only in terms of the MBP for linear moment functionals. However, Definition \ref{defn:dml} can be generalized.}
\begin{definition}[DML, following Definition 3.2 by { \color{black} ~\cite{10.1111/ectj.12097}}]\label{defn:dml}
Given a dataset ${\color{black}\dataset}$ of $n$ observations, split the dataset $\dataset$ into $k$ random disjoint subsets $\dataset_j$ of the same size. Consider a score function $\varphi( \theta, \boldsymbol \eta)$ that fulfills the MBP as in \eqref{eq:debiased_score}. Construct estimators $\hat{\boldsymbol{\eta}}_j = (\hat{\alpha}_j, \hat{g}_j)$ for the nuisance parameter of the score using datasets ${\color{black}\dataset}\setminus {\color{black}\dataset_j}$. Then, construct an estimator $\hat{\theta}$ of the paraemeter $\theta$ as the solution to the following equation:
\begin{equation*}
\label{eq:sol_model}
    k^{-1}\sum_{j=1}^k \hat{\mathbb{E}}_{{\color{black}\dataset_j}} \left [\varphi( \theta, \hat{\boldsymbol \eta}_j) \right ] = 0,
\end{equation*}
where $\hat{\mathbb{E}}_{{\color{black}\dataset_j}}[\ \cdot \ ]$ is the empirical expected value over ${\color{black}\dataset_j}$.
\end{definition}
\subsection{Granger Causality Implies True Causation under \ref{cond:1}-\ref{cond:3}}
\label{eq:causal_granger}
As discussed in Section \ref{sec:granger}, Granger causality does not imply true causality in general. In our case, however, Granger causality corresponds to true causation, as stated in the following result.
\begin{restatable}{theorem}{granger}
\label{thm:granger}
Consider a causal model as in \ref{cond:1}-\ref{cond:3}. Then, {\color{black}it holds $X^i_t\in \mathsf{pa}_T(\boldsymbol{Y})$ for some $t, T\in \mathbb{Z}$ if and only if (iff.) \eqref{eq:def_granger} holds. That is,} a time series $\boldsymbol{X}^i$ {\color{black}has a direct causal effect on} $\boldsymbol{Y}$ iff. $\boldsymbol{X}^i$ Granger causes $\boldsymbol{Y}$.
\end{restatable}
Proof of this result is given in Appendix \ref{app:granger}. Importantly, Theorem \ref{thm:granger} does not require causal faithfulness. We remark that \citet{DBLP:conf/nips/PetersJS13} shows that Granger causality implies true causation for fully autoregressive models (see also \citet{DBLP:conf/clear2/LoweMSW22}). This result is based on the identifiability of additive noise models, in which all relevant variables are observed \citep{DBLP:conf/uai/PetersMJS11}. Our framework, however, is more general than \citet{DBLP:conf/nips/PetersJS13,DBLP:conf/clear2/LoweMSW22}, since it allows for confounding among the covariates (see Figure \ref{fig:graph}). In the special case of a fully autoregressive model, Theorem \ref{thm:granger} is equivalent to previous results \citep{DBLP:conf/nips/PetersJS13,DBLP:conf/clear2/LoweMSW22}. 
\begin{algorithm*}[t]
    \caption{The \alg}
	\label{alg}
	\begin{algorithmic}[1] 
	\STATE split the dataset $\dataset$ into $k$ random disjoint subsets $\dataset_j$ of the same size;\vspace{8pt}\\
    \FOR{each dataset partition $j $} \label{alg:begin_first_step}
	\STATE estimate $\hat{\boldsymbol \eta} ^0_j \coloneqq (\hat{\alpha}^0_j, \hat{g}^0_j)$ on dataset $\dataset \setminus \dataset_j$, with $\hat{g}^0_j$ an estimate for $g^0_0$, and $\hat{\alpha}^0_j$ an estimate of the RR $\alpha^0_0$ of $m_0(\boldsymbol{V}; g^0)$ as in Theorem \ref{lemma:cond_variance};\label{alg:regression}
    \ENDFOR
    \STATE $\hat{\theta}^{0} \gets k^{-1}\sum_{j=1}^k \hat{\mathbb{E}}_{{\color{black}\dataset_j}} \left [m_0(\boldsymbol{V}; \hat{g}^0_j ) + \hat{\alpha}^0_j(\boldsymbol{X})\cdot(Y - \hat{g}^0_j(\boldsymbol{X})) \right ]$;\label{alg:end_first_step} \vspace{5pt}\\
	\FOR{each potential cause $\boldsymbol{X}^i$}
    \FOR{each dataset partition $j$}
    \STATE estimate $\hat{\boldsymbol \eta} ^i_j \coloneqq (\hat{\alpha}^i_j, \hat{g}^i_j)$ on dataset $\dataset \setminus \dataset_j$, with $\hat{g}^i_j$ an estimate for $g^i_0$, and $\hat{\alpha}^i_j$ an estimate of the RR $\alpha^i_0$ of $m_i(\boldsymbol{V}; g^i)$ as in Theorem \ref{lemma:cond_variance};\label{alg:end_second_step}    
    \ENDFOR
    \STATE $\hat{\theta}^i \gets k^{-1}\sum_{j=1}^k \hat{\mathbb{E}}_{{\color{black}\dataset_j}} \left [m_i(\boldsymbol{V}; \hat{g}^i_j ) + \hat{\alpha}^i_j(\boldsymbol{X})\cdot(Y - \hat{g}^i_j(\boldsymbol{X})) \right ]$;\label{alg:end_third_step}
    \STATE perform a paired Student's $t$-test to determine if $\hat{\theta}^{i} \approx \hat{\theta}^0$, and select time series $\boldsymbol{X}^i$ as a cause if the null-hypotheses is rejected;\label{alg:test}
    \ENDFOR \vspace{5pt}\\
   	\STATE \textbf{return} the selected time series;
   	\end{algorithmic}
\end{algorithm*}
\subsection{Testing Granger Causality with DML}
\label{sec:granger_DML}
Our approach to identifying potential causes consists of performing a statistical test to determine if Eq. \ref{eq:def_granger} holds. Due to Theorem \ref{thm:granger}, a straightforward approach would then consist of using a conditional independence test, to select or discard a time series $\boldsymbol{X}^i$ as a cause of the outcome $\boldsymbol{Y}$. However, conditional independence testing is challenging in high-dimensional settings. Furthermore, kernel-based conditional independence tests \citep{DBLP:conf/nips/FukumizuGSS07,DBLP:conf/uai/ZhangPJS11,DBLP:conf/nips/ParkM20,sheng2020distance} are computationally expensive. Instead, we provide a new statistical test based on DML. Our approach is based on the observation that under \ref{cond:1}-\ref{cond:3}, the condition in Eq. \ref{eq:def_granger} can be written in terms of simple linear moment functionals. The following theorem holds.
\begin{restatable}{theorem}{cond}
\label{lemma:cond_variance}
Consider the notation as in Eq. \ref{eq:def_granger}, and fix a time step $T$. For any square-integrable random variable $g^0({\color{black}\boldsymbol{X}^i_{<T}}, {\color{black}\boldsymbol{I}_{<T}^{\setminus i}})$, consider the moment functional $m_0(\boldsymbol{V}; g^0) \coloneqq Y_T \cdot g^0$. Similarly, for any square-integrable random variable $g^i({\color{black}\boldsymbol{I}_{<T}^{\setminus i}})$, consider the moment functional $m_i(\boldsymbol{V}; g^i) \coloneqq Y_T \cdot g^i$. Assuming \ref{cond:1}-\ref{cond:3}, $\boldsymbol{X}^i$ Granger causes $\boldsymbol{Y}$ iff. it holds
\begin{equation}
\label{eq:functional}
\expect{}{m_0(\boldsymbol{V}; g_0^0 )} - \expect{}{m_i(\boldsymbol{V}; g_0^i )} \neq 0,
\end{equation}
with $g_0^0(\boldsymbol{x}, \boldsymbol{i}) = \mathbb{E}[Y_T \mid {\color{black}\boldsymbol{X}^i_{<T}} = \boldsymbol{x}, {\color{black}\boldsymbol{I}_{<T}^{\setminus i}} = \boldsymbol i]$, and $g_0^i(\boldsymbol{i}) = \mathbb{E}[Y_T \mid {\color{black}\boldsymbol{I}_{<T}^{\setminus i}} = \boldsymbol i]$.
\end{restatable}
The proof is deferred to Appendix \ref{app:cond_variance}. Intuitively, the parameter $\theta \coloneqq \mathbb{E}[Y_T .\mathbb{E}[Y_T | I]]$ quantifies the cross-correlation between the target variable $Y_T$ and its conditional mean $\mathbb{E}[Y_T | I]$, given the set of variables $I$. In general terms, variations in the parameter $\theta$ due to changes in the set $I$ suggest the presence of a causal relationship between alterations in the set $I$ and the target variable $Y_T$.

By this theorem, we can identify the causal parents of $Y$ by testing Eq. \ref{eq:functional}. This boils down to learning parameters $\theta^0_0 \coloneqq \expect{}{m_0(\boldsymbol{V}; g_0^0 )}$, $ \theta^i_0\coloneqq \expect{}{m_i(\boldsymbol{V}; g_0^i )} $, and then performing a sample test to verify if it holds $\theta^0 - \theta^i \neq 0$. Since both $m_0(\boldsymbol{V}; g_0^0 )$ and $m_i(\boldsymbol{V}; g_0^i )$ {\color{black}are linear moment functionals,} we can use DML as in Definition \ref{defn:dml} to estimate $\theta^0_0$ and $\theta^i_0$. Under mild convergence conditions on the nuisance parameters for $m_0(\boldsymbol{V}; g_0^0 )$ and $m_i(\boldsymbol{V}; g_0^i )$, DML ensures fast convergence in distribution to the true parameters.

\section{The Algorithm}
\label{sec:approach}
\subsection{Overview} 
Our algorithm learns causal relationships between time series, by testing Granger causality using DML, as outlined in Section \ref{sec:granger_DML}. We refer to our approach as the \alg (Structure Identification from Temporal Data). Our method is presented in Algorithm \ref{alg}. This algorithm essentially performs the following steps:
\begin{enumerate}[leftmargin=7mm,label={(\arabic*)},itemsep=-1ex,topsep=-1ex]
    \item \label{step:1} Select a potential cause $\boldsymbol{X}^i$ to test if $\boldsymbol{X}^i$ Granger causes $\boldsymbol{Y}$. Split the dataset $\dataset$ into $k$ disjoint sets $\dataset_j$ with $k \geq 2$.
    \item Estimate $\hat{\boldsymbol \eta} ^0_j \coloneqq (\hat{\alpha}^0_j, \hat{g}^0_j)$ on dataset $\dataset \setminus \dataset_j$, with $\hat{g}^0_j$ an estimate for $g^0_0$, and $\hat{\alpha}^0_j$ an estimate of the RR $\alpha^0_0$ of $m_0(\boldsymbol{V}; g^0)$ as in Theorem \ref{lemma:cond_variance}. Similarly, provide an estimate $\hat{\boldsymbol \eta} ^i_j \coloneqq (\hat{\alpha}^i_j, \hat{g}^i_j)$ on dataset $\dataset \setminus \dataset_j$ for the pair $\boldsymbol \eta^i_0 = (g^i_0, \alpha^i_0)$, with $g^i_0$ as in Theorem \ref{lemma:cond_variance} and $\alpha^i_0$ the RR $\alpha^i_0$ of $m_i(\boldsymbol{V}; g^i)$. This step corresponds to Line 3 and Line 8 of Algorithm \ref{alg}.
    \item \label{step:3} Provide an estimate $\hat{\theta}^0 \approx \mathbb{E}[m_0(\boldsymbol{V}; g^0)]$, by solving the equation $k^{-1}\sum_{j=1}^k \hat{\mathbb{E}}_{{\color{black}\dataset_j}} \left [\varphi_0( \theta, \hat{\boldsymbol \eta}_j) \right ] = 0$ with a score of the form $\varphi_0( \theta, \hat{\boldsymbol \eta}_j) \coloneqq m_0(\boldsymbol{V}; \hat{g}^0_j ) + \hat{\alpha}^0_j(\boldsymbol{X})\cdot(Y - \hat{g}^0_j(\boldsymbol{X})) - \theta$. This step corresponds to Line 5 of Algorithm \ref{alg}. Provide an estimate $\hat{\theta}^i \approx \mathbb{E}[m_i(\boldsymbol{V}; g^i)]$ in a similar fashion, as in Line 10 of Algorithm \ref{alg}. 
    \item \label{step:4} Use a {\color{black}paired} Student's $t$-test to determine if $\expect{}{\theta^0 - \theta^i}$ is approximately zero. Select $\boldsymbol{X}^i$ as a cause of $\boldsymbol{Y}$ if the null hypothesis is rejected. This step corresponds to Line 11 and 8 of Algorithm \ref{alg}.
\end{enumerate}

\subsection{Strong Consistency Guarantees} 
In this paragraph, we provide an explanation for Step \ref{step:4} of our algorithm. Under mild structural conditions on $g^0_j,g^i_j$ and $\alpha^0_j,\alpha^i_j$ \citet{DBLP:conf/icml/ChernozhukovNQS22,10.1111/ectj.12097,10.1093/biomet/asaa054}, the quantity $\theta^0 - \theta^i$ has $\sqrt{n}$-consistency. Hence, it holds

\begin{equation}
\label{eq:consistency}
\sqrt{n}\left(\hat{\theta}^0 - \hat{\theta}^i\right) \xrightarrow[]{d} \mathcal{N}\left(0, \sigma^2\right),
\end{equation}
if and only if $\expect{}{m_0(\boldsymbol{V}; g_0^0 )} - \expect{}{m_i(\boldsymbol{V}; g_0^i )} = 0$. Then, by Theorem \ref{lemma:cond_variance}, Eq. \ref{eq:consistency} holds iff. $\boldsymbol{X}^i$ does not Granger causes $\boldsymbol{Y}$. The notation in \eqref{eq:consistency} means that the difference $\hat{\theta}^0 - \hat{\theta}^i$ converges in distribution to a zero-mean Gaussian for an increasing number of samples. That is, for all $\epsilon>0$ and $\zeta>0$ there exists  $\delta>0$, such that given $n>\delta$ samples it holds
$\mathbb{P}\left(\lvert \hat{\theta}^0 - \hat{\theta}^i\rvert >\epsilon\right)\leq 1-\Phi\left(\epsilon\frac{\sqrt{n}}{\sigma}\right)+\frac{\zeta}{2}$, if and only if $\expect{}{m_0(\boldsymbol{V}; g_0^0 )} - \expect{}{m_i(\boldsymbol{V}; g_0^i )} = 0$. Here, $\Phi$ is the CDF of the standard Normal distribution. By this inequality, we can use a {\color{black} paired} Student's $t$-test to determine if $\boldsymbol{X}^i$ Granger causes $\boldsymbol{Y}$, as in Line 11 of Algorithm \ref{alg}. We refer the reader to \citet{DBLP:conf/icml/ChernozhukovNQS22,10.1111/ectj.12097,10.1093/biomet/asaa054}
for a detailed survey on strong consistency for DML and its relationship with the MBP.
\subsection{ Complexity of Algorithm \ref{alg}} 
Much of the run time of our algorithm consists of performing a regression to learn $\boldsymbol \eta^0_j$ and $\boldsymbol \eta^i_j$ on dataset $\dataset_j$. Denote with $d$ the time complexity of performing such a regression. For a given $k$-partition of the dataset and a fix potential cause $\boldsymbol X^i$, we can upper-bound the time complexity of our algorithm as $\mathcal{O}(dk)$. {\color{black}Furthermore, since a regression to learn $\boldsymbol \eta^i_j$ is performed for each potential cause, i.e., $m$ times, the complexity of the algorithm is $\mathcal{O}(dkm)$.} Here, $d$ depends on the specific techniques used for the regression. Non-parametric regression can be performed efficiently in the problem size. We analyse the computational complexity in Appendix \ref{app:comp} further and show runtime plots in \ref{sec:runtime}. We further discuss the run time and extension for full causal discovery in Appendix \ref{sec:full}.

We further improve efficiency in practice as follows: Instead of computing $\boldsymbol \eta_0^i = (g_0^i,\alpha_0^i)$ directly, we apply a zero-masking layer to the NNs used to estimate $g^0_j$ and $\alpha^0_j$ for the features $\boldsymbol{X}^i$. This masking layer tells the sequence-processing layers that the input values for features $\boldsymbol{X}^i$ should be skipped. We then compute $\theta^i_0$ using the resulting surrogate models $\tilde{g}^i_j$ and $\tilde{\alpha}^i_j$. Please refer to~\cref{app:zero-mask} on an intuition behind zero-masking, and why zero-masking may not hurt the estimations. Using zero-masking dramatically improves run time, since it allows to perform only two regressions through the entire run time of Algorithm \ref{alg}.

{\color{black}
\subsection{Practical implementation at a glance}
When deploying \alg, we follow the workflow below.
\begin{itemize}[leftmargin=3mm,itemsep=-1ex,topsep=-1ex]
    \item \textbf{Single model for both nuisance parameters.}  
    In our setting the Riesz representer coincides with the regression nuisance, because $\mathbb{E}[m(\mathbf{V};g)] \;=\; 
        \mathbb{E}_{\mathbf{X}}\!\bigl[\mathbb{E}[Y \mid \mathbf{X}]\,g\bigr]
        \;\;\Longrightarrow\;\;
        a(\mathbf{X}) \;=\; \mathbb{E}[Y \mid \mathbf{X}] \;=\; g(\mathbf{X})$.
    Hence we train one regressor and reuse it for both quantities; training two identical copies yields indistinguishable results.
    \item \textbf{Choice and tuning of the regressor.}  
    \alg\ is agnostic to the model class. A small validation split is reserved to select the family (kernel, MLP, \dots) and its hyper-parameters. For the synthetic datasets (Sec. \ref{sec:synthetic_experiments}), abundant samples justify a fixed two-layer MLP throughout. For DREAM3 (Sec. \ref{sec:exp_semi}), we tune on the first sub-task (\emph{E.\,Coli 1}) and adopt \texttt{KernelRidge} with a third-degree polynomial kernel (\texttt{alpha = 1}, \texttt{coef0 = 1}) for all remaining tasks.
%
    \item \textbf{Lag selection.}  
    Lag is treated as a hyper-parameter. For the \textit{synthetic datasets} (\ref{sec:synthetic_experiments}), we use the ground-truth lag employed in data generation. For \textit{DREAM3} (\ref{sec:exp_semi}) we fix $\textit{lag} = 2$ for every method, following prior work. However note that stationarity is \emph{not} required by our theory; if it fails, the Granger-causality condition must simply be checked at all time points rather than a single $T$.
%
    \item \textbf{Cross-fitting scheme.}  
    We employ $k = 5$-fold cross-fitting, splitting trajectories uniformly at random into
    equal-sized folds—an essential step for valid double-machine-learning inference.
\end{itemize}
}
\section{Experiments}
\begin{table*}[tb]
\centering
\caption{AUROC score for the \alg and common algorithms for Granger causality. Models with "$^\textrm{\ding[1]{70}}$" sign use deep neural networks. When reported, the error bounds represent 1 standard deviation of the AUROC score over 5 experiment repetitions}
{\scriptsize
\begin{tabular}{lccccc}
\toprule
\textbf{Method}  & \textbf{E.Coli 1} & \textbf{E.Coli 2} & \textbf{Yeast 1} & \textbf{Yeast 2} & \textbf{Yeast 3} \\
\midrule
cMLP$^\textrm{\ding[1]{70}}$   & 0.644    & 0.568    & 0.585   & 0.506   & 0.528   \\
cLSTM$^\textrm{\ding[1]{70}}$    & 0.629    & 0.609    & 0.579   & 0.519   & 0.555   \\
TCDF$^\textrm{\ding[1]{70}}$   & 0.614    & 0.647    & 0.581   & 0.556   & 0.557   \\
SRU$^\textrm{\ding[1]{70}}$     & 0.657    & 0.666    & 0.617   & 0.575   & 0.55    \\
eSRU$^\textrm{\ding[1]{70}}$     & 0.66     & 0.629    & 0.627   & 0.557   & 0.55    \\
DYNO. &0.590 &0.547 &0.527 &0.526 &0.510\\
PCMCI$^+$ & 0.530 & 0.519 & 0.530 & 0.510 & 0.512 \\
Rhino$^\textrm{\ding[1]{70}}$ Reprod.  & 0.671 $\pm0.014$      & 0.640 $\pm 0.022$        & \textbf{0.656} $\pm 0.011$        & 0.565 $\pm 0.012$         & 0.549 $\pm 0.004$\\
Rhino+g$^\textrm{\ding[1]{70}}$ Reprod.  & 0.665 $\pm 0.023$         & 0.646 $\pm 0.032$         &0.649 $\pm 0.011$         &  	0.582 $\pm 0.011$        & \textbf{0.571} $\pm 0.010$\\
\textbf{DR-SIT (ours)}   & \textbf{0.704}$\pm0.005$         & \textbf{0.680}$\pm 0.004$        &  0.653$\pm 0.001$         & \textbf{0.585}$\pm 0.003$        & 0.544$\pm 0.003$\\     
\bottomrule
\end{tabular}
}
\label{table: Exp DREAM3 AUROC}
\end{table*}

In this section we provide an overview of the main experimental results. \textbf{We provide additional extensive experiments in Appendix \ref{sec:appendix_experiments}.}
{
\color{black}
\subsection{Synthetic Experiments}
\label{sec:synthetic_experiments}
%
\noindent\textbf{Dataset generation.} We first set the number of potential causes $m$, and we fix the lag $\Delta$ and the number of time steps for the dataset $T$. We create a covariate adjacency 3-dimensional tensor $\Sigma$ of dimensions $\Delta \times m \times m$. This tensor has $0$-$1$ coefficients, where $\Sigma_{k,i,j} = 1$ if $X^i_{t - k}$ has a casual effect on $X^j_{t}$ for all time steps $t$. Similarly, we create an adjacency matrix $\Sigma^{\boldsymbol{Y}}$ for the outcome $\boldsymbol{Y}$. $\Sigma^{\boldsymbol{Y}}$ is a binary $\Delta \times m $ array, such that $\Sigma^{\boldsymbol{Y}}_{k,j}=1$ if $X_{t-k}^{j}$ has a causal effect on $Y_{t}$ for all time steps $t$. The entries of $\Sigma$ and $\Sigma^{\boldsymbol{Y}}$ follow the Bernouli distribution with parameter $p=0.5$. Note that the resulting casual structure fulfills \ref{cond:2}-\ref{cond:3}. 

We then create $m$ transformations that are used to produce the potential causes $\boldsymbol{X}^1, \dots, \boldsymbol{X}^m$. Each one of these transformations is modeled by an \textsc{MLP} with $1$ hidden layer and $200$ hidden units. We use \textsc{Tanh} nonlinearities (included also in the output layer) in order to control the scale of the values. The final output value is further scaled up so that all transforms generate values in the range $[-10,10]$. The input layer of each MLP is coming from the corresponding causal parents of the corresponding time series, as calculated from $\Sigma$. 

In order to generate the potential causes $\boldsymbol{X}^i$, we use $\Sigma$ and the MLP transforms. The first value of each time series is randomly generated from a uniform distribution in $[-10,10]$. Then, each $\boldsymbol{X}^{i}$ is produced by applying the appropriate transform to its causal parents, as determined by $\Sigma$, and a zero-mean unit variance Gaussian noise is added. We generate the target time series in a similar fashion. Each variable $Y_{t}$ is produced by applying the MLP transform to its causal parents, as determined by the target adjacency matrix $\Sigma^{\boldsymbol{Y}}$. We then add zero-mean Gaussian noise. The scale of the posterior additive noise for the outcome is referred to as the \emph{noise-to-signal ratio} (NTS).

\noindent\textbf{Description of the experiments.} We are given a dataset as described above with $m$ potential causes and a fixed NTS for the generation of the outcome $\boldsymbol{Y}$. We determine which series $\boldsymbol{X}^1, \dots, \boldsymbol{X}^n$ are the causal parents of the outcome, using Algorithm \ref{alg}. For a given choice of $m$ and NTS, we repeat the runs five times. This experiment is repeated for an increasing number of potential causes, and increasing noise-to-signal ratio, to evaluate the performance of Algorithm \ref{alg} on challenging datasets.

In this set of experiments, we learn $\eta_j^i$ as in Line \ref{alg:regression} of Algorithm \ref{alg}, using for the regression task an MLP model. We found that this simple approach, combined with zero-masking, dramatically reduces the false positives of the Student's t-test in Line \ref{alg:test} of Algorithm \ref{alg}.

\noindent\textbf{Results}
\begin{figure*}[tb]

    \makebox[490pt][c]{
    \begin{subfigure}[t]{1.8in}
        \centering
        \includegraphics[width=1.8in]{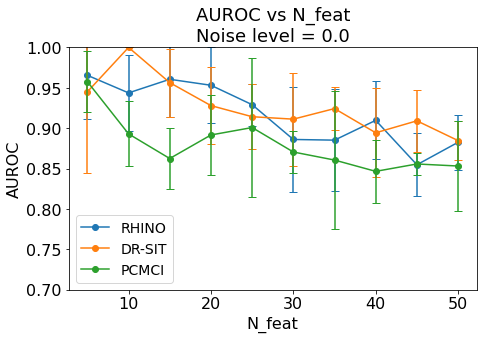}
    \end{subfigure}
       \begin{subfigure}[t]{1.8in}
        \centering
        \includegraphics[width=1.8in]{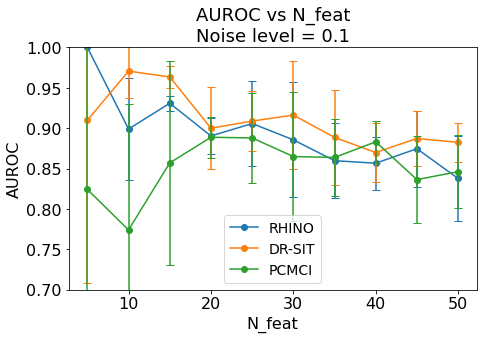}
    \end{subfigure}
   \begin{subfigure}[t]{1.8in}
        \centering
        \includegraphics[width=1.8in]{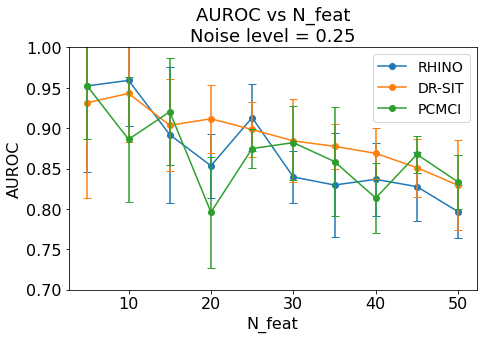}
    \end{subfigure}
    }
 \caption{AUROC metric for DR-SIT, RHINO and PCMCI for various noise levels on the synthetic dataset.}\label{fig:AUROC_synthetic_dataset_small}
\end{figure*}
In Figure~\ref{fig:AUROC_synthetic_dataset_small}) (also shown in greater detail in Appendix Figure~\ref{AUROC_synthetic_dataset}) we plot the AUROC performance of DR-SIT against RHINO and PCMCI on the synthetic dataset, confirming the competitive performance of DR-SIT. In order to calculate the AUROC for DR-SIT, we sort our predictions (for existence of an edge)  on
   the 
   standard deviation of 
   $Z:= Y_T \cdot \tilde{g}_j^i + \tilde{\alpha}_j^i \cdot (Y_T - \tilde{g}_j^i ) - Y_T \cdot g_j^0 - \alpha_j^0 \cdot (Y_T - g_j^0 )$, which is simply the difference of the doubly robust statistics for $\theta^i$ and $\theta^0$ for a datapoint in partition $D_j$. Moreover, in Appendix \ref{sec:appendix_synthetic}, we show the performance of our method with respect to the accuracy, F1 and CSI scores (see Tablse \ref{table:snthetic1}-\ref{table:snthetic2}). We see that the \alg is stable for an increasingly higher posterior noise and scales reasonably w.r.t the number of observations. 

\if{0}
\begin{figure}[t]
    \centering
    \includegraphics[width=5in]
    {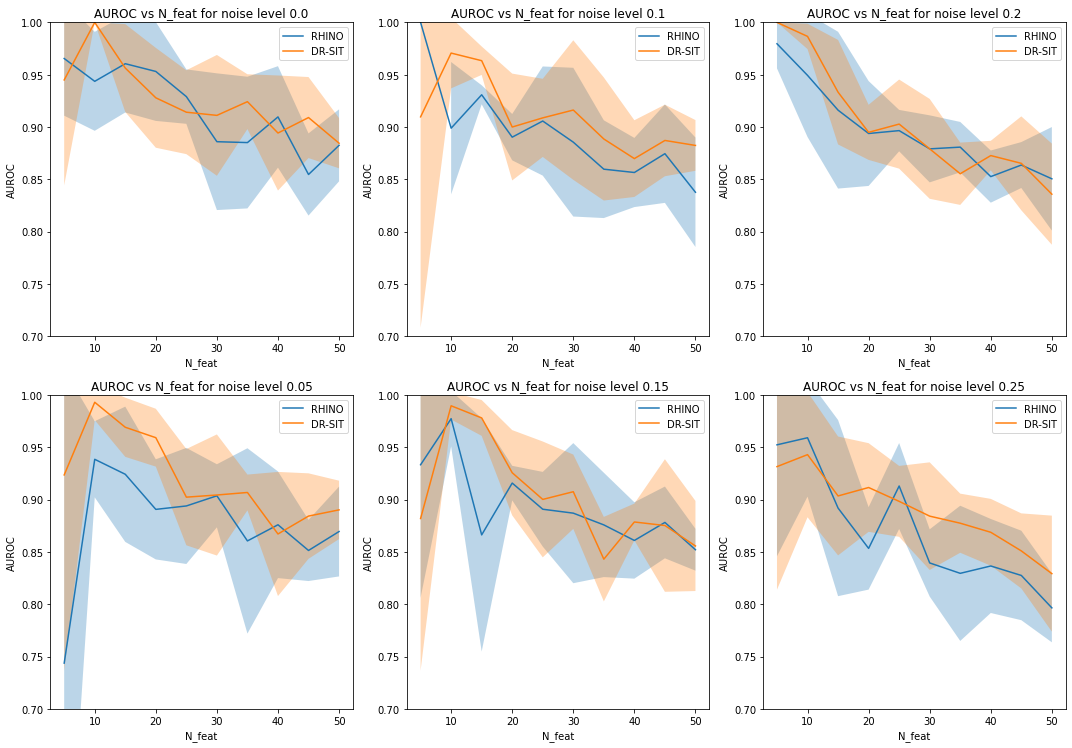}
    \caption{AUROC metric for RHINO and DR-SIT for various noise levels on the synthetic dataset.}
    \label{AUROC_synthetic_dataset}
\end{figure}
\fi
}
\subsection{Semi-Synthetic Experiments} \label{sec:exp_semi}
\noindent\textbf{The Dream3 dataset.} Following previous related work \citep{tank2018neural,khanna2019economy,DBLP:journals/make/NautaBS19,bussmann2021neural,DBLP:journals/corr/abs-2210-14706}, we evaluate performance with the Dream3 benchmark~\citep{dream31,dream32}. The Dream3 benchmark is a collection of gene expression level measurements across five different networks, where each network comprises 100 genes. The data in Dream3 are organized as time series. Specifically, for each of the five networks 46 trajectories are given; each trajectory details the progression of the 100 genes over 21 time steps after an initial perturbation in the gene values.

\noindent\textbf{Description of the experiments.} Our goal is to infer the causal structure of each network. This gives us a total of five tasks, i.e., E.Coli 1, E.Coli 2, Yeast 1, Yeast 2, and Yeast 3. We run our algorithm on this dataset and we use the area under the ROC curve (AUROC) as the performance metric. Following ~\citep{DBLP:journals/corr/abs-2210-14706}, we compare against the following baselines: cMLP~\citep{tank2018neural}, cLSTM~\citep{tank2018neural}, TCDF~\citep{DBLP:journals/make/NautaBS19}, SRU~\citep{khanna2019economy}, eSRU~\citep{khanna2019economy}, Dynotears~\citep{pamfil2020dynotears}, Rhino+g~\citep{DBLP:journals/corr/abs-2210-14706}, and Rhino~\citep{DBLP:journals/corr/abs-2210-14706}. 

\begin{figure*}[tb]

    \makebox[490pt][c]{
    \begin{subfigure}[t]{1.4in}
        \centering
        \includegraphics[width=1.4in]{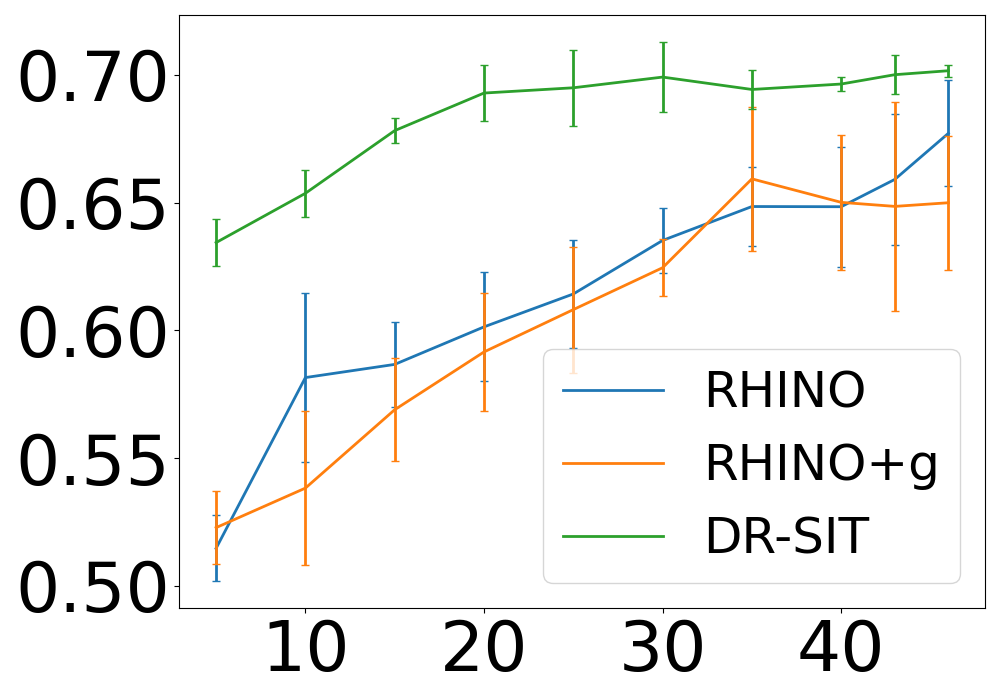}
        \caption{Task: E.Coli 1}\label{fig:1A}       
    \end{subfigure}
       \begin{subfigure}[t]{1.4in}
        \centering
        \includegraphics[width=1.4in]{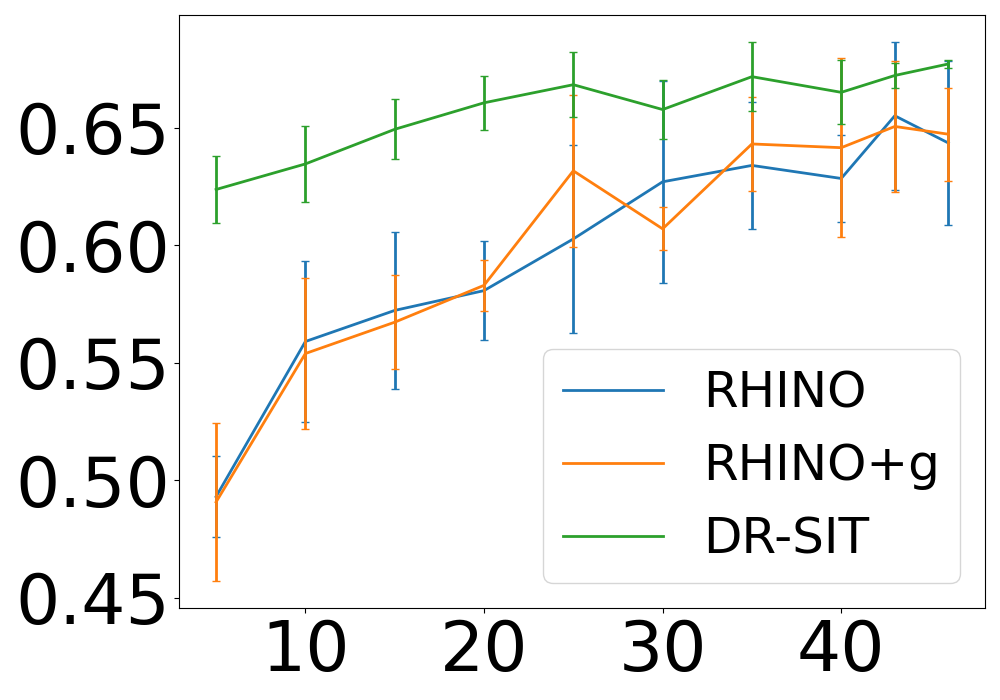}
        \caption{Task: Ecoli 2}\label{fig:1B}    
    \end{subfigure}
   \begin{subfigure}[t]{1.4in}
        \centering
        \includegraphics[width=1.4in]{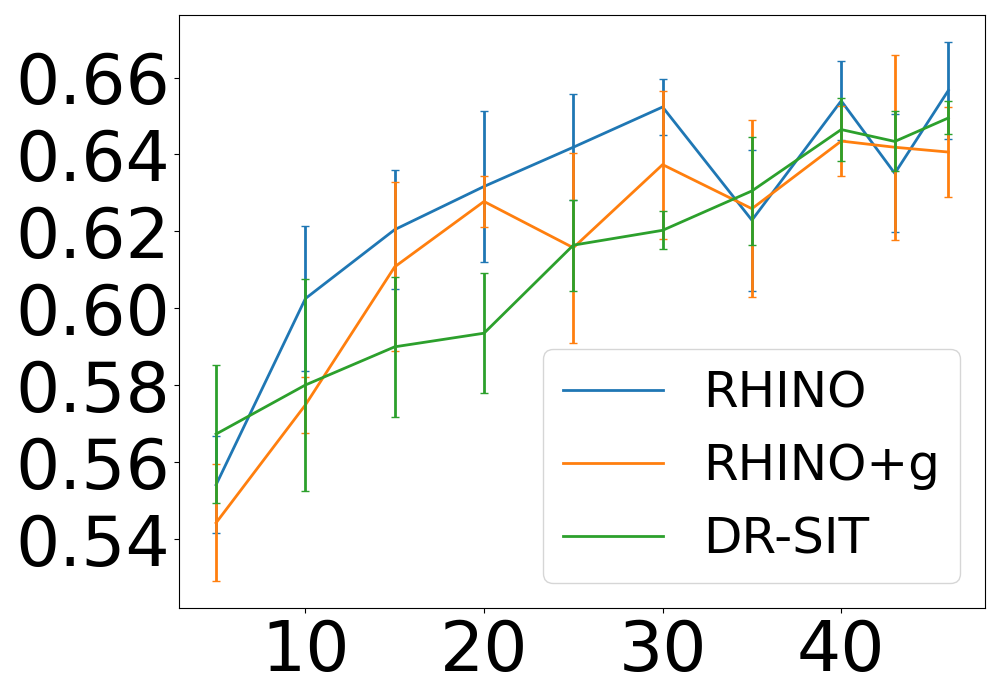}
        \caption{Task: Yeast 1}\label{fig:1C}    
    \end{subfigure}
   \begin{subfigure}[t]{1.4in}
        \centering
        \includegraphics[width=1.4in]{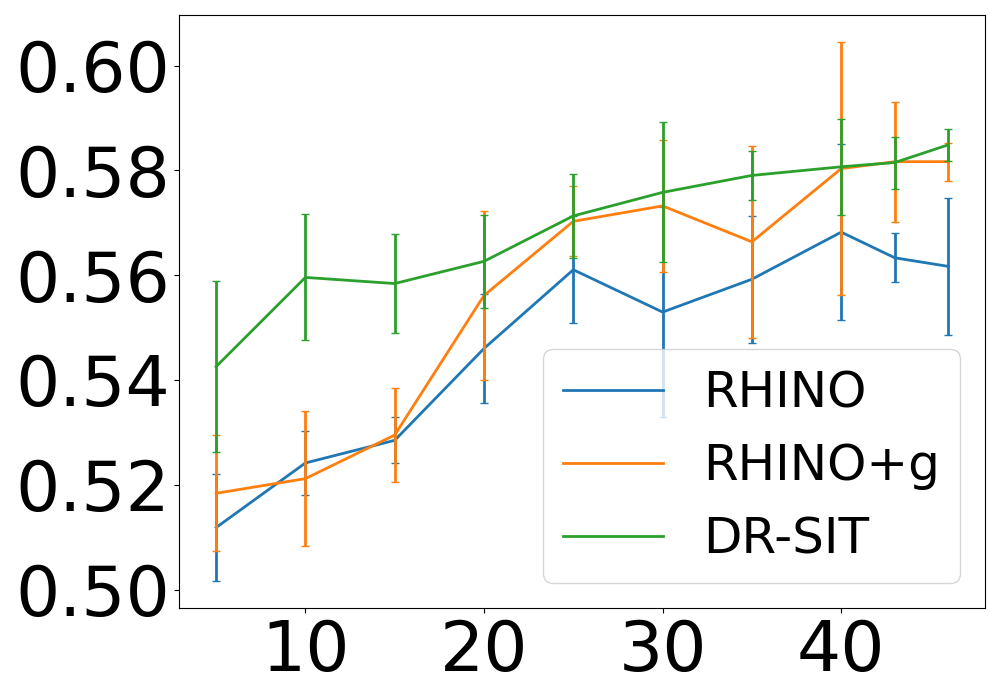}
        \caption{Task: Yeast 2}\label{fig:1D}    
    \end{subfigure}
    
    }
 \caption{This figure demonstrates the consistent performance of DR-SIT w.r.t number of training observations (i.e trajectories) compared to state-of-the-art methods Rhino and Rhino+g. For the same plots in all 5 tasks depicted in greater resolution, see \cref{fig:exp_observations} in \cref{sec:low_sample}. 
}\label{fig:low_sample_regime}
\end{figure*}

\begin{table*}[tb]
\centering
\caption{Run time and Hardware used for our method (DR-SIT) and the state-of-the-art baseline Rhino.}
{\scriptsize
\begin{tabular}{l c c}
\toprule
\textbf{Category\textbackslash Method}  & \textbf{Rhino} & \textbf{DR-SIT (ours)} \\
\midrule
Runtime  & 18 mins 40 sec $\pm$ 30 sec & 57.12 sec $\pm$ 1.6 sec \\    
\midrule
Hardware & 1 NVIDIA A100 GPU + AMD EPYC 7402 24-Core CPU & 11th Gen Core i5-1140F CPU \\
\bottomrule
\end{tabular}
}
\label{table:runtime}
\end{table*} 

\noindent\textbf{Results.} Shown in ~\cref{table: Exp DREAM3 AUROC}.
The results for cMLP, cLSTM, TCDF, SRU and SRU are taken directly from \citet{khanna2019economy,DBLP:journals/corr/abs-2210-14706}, where error bounds are not reported. Regarding Rhino+g and Rhino, we partially reproduce the experiments by \citet{DBLP:journals/corr/abs-2210-14706}, using their source code. Our implementation of Rhino+g and Rhino differs from \citet{DBLP:journals/corr/abs-2210-14706} only in the choice for the hyper-parameters, which is the same on all five tasks. We specifically use the following hyper-parameters for Rhino: Node Embedd. = $16$, Instantaneous eff. = False, Node Embedd. (flow) = $16$, lag = $2$, $\lambda_s$ = $19$, Auglag = $30$. And we use the following for Rhino+g: Node Embedd. = $16$, Instantaneous eff. = False, lag = $2$, $\lambda_s$ = $15$, Auglag = $60$. This is the setting that is reported for the Ecoli1 subtask and found in their corresponding code implementation. We opted for this approach because in the experiment by \citet{DBLP:journals/corr/abs-2210-14706}, it seems that Rhino overfitted to the dataset. We learn $\eta^i_j$ as in Line \ref{alg:regression} of Algorithm \ref{alg}, using a simple kernel ridge regression model with polynomial kernels of degree three combined with zero masking. 

We observe that \alg outperforms all the other benchmarks on three tasks (E.Coli 1, E.Coli 2, Yeast 2). Furthermore, \alg obtains comparable performance on the remaining tasks (Yeast 1,Yeast 3). We also would like to emphasize that our estimator is much simpler than deep nets such as Rhino or Rhino+g. As such, it has lower sample complexity and lower run time than the other algorithms.

\subsection{Low sample regime}
One appealing property of DR-SIT is the strong consistency of the estimators (see Section \ref{sec:approach}). Due to this property, the estimators $\hat{\theta}^0$ and $\hat{\theta^i}$ exhibit fast convergence to the true parameters. Hence, our algorithm enjoys competitive results in low sample complexity settings, than richer models such as Rhino \citep{DBLP:journals/corr/abs-2210-14706}.
%
We illustrate a compelling example of this, by comparing Rhino and Rhino+g against \alg. In this example, \alg learns nuisance parameters (Line 3 and Line 8 of Algorithm \ref{alg}) using a simple kernel ridge regression estimator with polynomial kernels of degree three. In  \cref{fig:low_sample_regime} we explicitly compare the performance of \alg, Rhino, and Rhino+g as the number of trajectories used in training varies and demonstrate significantly improved performance especially in low sample settings. 

\subsection{Run Time and Hardware}
\label{sec:runtime}

We report on the average runtime of DR-SIT and the competitive rival Rhino for experiment~\cref{table: Exp DREAM3 AUROC} across all five tasks (E.Coli 1, E.Coli 2, Yeast 1, Yeast 2 and Yeast 3) in~\Cref{table:runtime}. Despite having access only to a single CPU in contrast to the GPU-equipped execution of Rhino our method is almost 20x faster. This is because of the fact that DR-SIT algorithm will provide reasonable results even when employing simple fast efficient estimators (in this case a kernel regression). For a visual presentation of the AUROC performance progression vs runtime for various combinations of tasks and training sample sizes, see \cref{fig:auroc_vs_time_Ecoli1,fig:auroc_vs_time_Ecoli2,fig:auroc_vs_time_Yeast1,fig:auroc_vs_time_Yeast2,fig:auroc_vs_time_Yeast3} in appendix.

\section{Discussion}
In this work, we propose an efficient algorithm for doubly robust structure identification from temporal data.  We further provide asymptotical guarantees that our method is able to discover the direct causes even when there are cycles or hidden confounding and that our algorithm has $\sqrt{n}$-consistency. We extensively discuss the relations of the approach between the popular frameworks of Granger and Pearl's causality as well as relate and extend approaches from debiased machine learning to structure discovery from temporal data. We hope that our approach enables important real-world applications in bio-medicine where robustness to confounding, sample efficiency, and ease of use are important for causal discovery from observational time series.

\bibliography{main}
\bibliographystyle{tmlr}

\newpage
\appendix

\section{Direct Causal Effects and Interventions}
\label{app:interventions}

In this section, we clarify the notion of an intervention and direct causal effects. We will also introduce some notation that will be used later on in the proofs.

We consider interventions by which a random variable $X^j_t$ is set to a constant $X^j_t \gets x$. We denote with $Y_{T} \mid do(X^i_t = x)$ the outcome time series $\boldsymbol{Y}$ at time step $T$, after performing an intervention as described above. We can likewise perform multiple joint interventions, by setting a group of random variables $\boldsymbol I $ at different time steps, to pre-determined constants specified by an array $\boldsymbol i$. We use the symbol $Y_T \mid do(\boldsymbol I = \boldsymbol i)$ to denote the resulting post-interventional outcome, and we denote with $\pr{Y_T = y \mid do(\boldsymbol I = \boldsymbol i)}$ the probability of the event $\{Y_T \mid do(\boldsymbol I = \boldsymbol i) = y\}$.

Using this notation, a time series $\boldsymbol{X}^i$ has a direct effect on the outcome $\boldsymbol{Y}$, if performing different interventions on the variables $\boldsymbol{X}^i$, while keeping the remaining variables fixed, will alter the probability distribution of the outcome $\boldsymbol{Y}$. Formally, define the sets of random variables $\boldsymbol{I}_{T}\coloneqq \{X_{t}^1, \dots, X_{t}^n , Y_t\}_{ t < T}$, which consists of all the information before time step $T$. Similarly, define the random variable $\boldsymbol{X}^i_T \coloneqq \{X^i_t\}_{t < T}$, consisting of all the information of time series $\boldsymbol{X}^i$ before time step $T$. Define the variable $\boldsymbol{I}_{T}^{\setminus i}\coloneqq \boldsymbol{I}_{T} \setminus \boldsymbol{X}^i_T$, which consists of all the variables in $\boldsymbol{I}_T$ except for $\boldsymbol{X}^i_T$. Then, a time series $\boldsymbol{X}^i$ has a direct effect on the outcome $ \boldsymbol{Y}$ if it holds
\begin{equation}
\label{eq:direct_effect}   
\pr {Y_T = y \mid do \left (\boldsymbol{X}^i_T = \boldsymbol{x}' , \boldsymbol{I}_{T}^{\setminus i} =  \boldsymbol{i} \right )} \neq \pr {Y_T = y \mid do \left (\boldsymbol{X}^i_T = \boldsymbol{x}'' , \boldsymbol{I}_{T}^{\setminus i} =  \boldsymbol{i} \right )}
\end{equation}
We say that a time series $\boldsymbol{X}^i$ causes $\boldsymbol{Y}$, if there is a direct effect between $\boldsymbol{X}^i$ and $\boldsymbol{Y}$ as in Eq. \ref{eq:direct_effect}, for any time step $T$.
%
%
\section{Necessity of the Statistical Independence of $\varepsilon_t$}
\label{sec:counterexample}
We provide a counterexample, to show that if there are dependencies between the noise and the historical data, then the causal structure may not be identifiable from observational data. To this end, we consider a first dataset $\{X_t, Y_t\}_{t \in \mathbb{Z}}$, defined as
\begin{equation*}
\begin{array}{l}
X_{t-1}, Y_t \sim \mathcal{N}(\boldsymbol 0, \Sigma) 
\end{array}
\end{equation*}
Here, $\mathcal{N}(\boldsymbol 0, \Sigma)$ is a zero-mean joint Gaussian distribution with covariance matrix
\begin{equation*}
\Sigma = \left [
\begin{array}{cc}
1 & 1 \\
1 & 1
\end{array}
\right ].
\end{equation*}
We also consider a second dataset $\{W_t, Z_t\}_{t \in \mathbb{Z}}$, defined as 
\begin{equation*}
W_t \sim \mathcal{N}(0, 1), \quad Z_t = W_{t-1} 
\end{equation*}
The parameter $\Sigma$ is defined as above. Both datasets entail the same joint probability distribution. However, the causal diagrams change from one dataset to the other. Hence, the causal structure cannot be recovered from observational data, if the posterior additive noise $\varepsilon_t$ is correlated with some of the covariates.

\section{Necessity of No Instantaneous Causal Effects between $\boldsymbol{Y}$ and the Potential Causes $\boldsymbol{X}^i$} \label{app:instantaneous}

Here, we provide a counterexample to show that without the no instantaneous causal effect, the causal structure may not be identifiable from observational data. Consider the following two models:

\begin{itemize}
    \item Model 1: We consider time series $\{X_t\}$, $\{Y_t\}$ of the form $X_t = \mathbb{E}[X_{t-1}] + c$ and $Y_t = \mathbb{E}[Y_{t-1}-X_{t-1}] + X_t$. In this model, $c$ is a random variable drawn from a Gaussian distribution with a mean of 0 and covariance of 1.

    \item Model 2: We consider time series $\{X_t\}$, $\{Y_t\}$ of the form $Y_t = \mathbb{E}[Y_{t-1}] + c$ and $X_t = \mathbb{E}[X_{t-1}-Y_{t-1}] + Y_t$. In this model, $c$ is a random variable drawn from a Gaussian distribution with a mean of 0 and covariance of 1.
\end{itemize}
Both models entail the same joint distribution. However, in Model 1, $X$ has a causal effect on $Y$, whereas in Model 2, $Y$ has a causal effect on $X$. Hence, in this example, the causal structure is not identifiable.

\section{Proof of Theorem \ref{thm:granger}}
\label{app:granger}
We prove the following result.
\granger*
\begin{proof}
We first prove that it holds
\begin{equation}
\label{eq:2_new}
    \pr{Y_T = y \mid do(X^i_{t} = x,\boldsymbol{I}_{T}^{\setminus i} = \boldsymbol{i})} = \pr{Y_T = y \mid X^i_{t} = x,\boldsymbol{I}_{T}^{\setminus i} = \boldsymbol{i}}
\end{equation}
for any non-zero event $\{ Y_T = y \}$. To this end, define the group $\boldsymbol{P} $ consisting of all the causal parents of the outcome. Note that $\boldsymbol{P} \subseteq \{X^i_t, \boldsymbol{I}_{T}^{\setminus i}\} $. By \ref{cond:1}, the outcome can be described as $Y = f(\boldsymbol{P}) + \varepsilon$, where $\varepsilon$ is independent of $\{X^i_t, \boldsymbol{I}_{T}^{\setminus i} \}$. Hence by Rule 2 of the do-calculus (see \cite{pearlj}, page~85) Eq. \ref{eq:2_new} holds, since $Y$ becomes independent of $\{X^i_t, \boldsymbol{I}_{T}^{\setminus i} \}$ once all arrows from $\boldsymbol{P}$ to $Y$ are removed from the graph of the DGP.

We now prove the claim using Eq. \ref{eq:2_new}. To this end, assume that Eq. \ref{eq:2_new} holds and suppose that $X^i$ does not Granger causes $Y$, i.e., it holds
\begin{equation}
\label{eq:granger}
\pr{Y_T=y\mid X^i_{t} = x,\boldsymbol{I}_{T}^{\setminus i} = \boldsymbol{i}} = \pr{Y_T=y\mid  \boldsymbol{I}_{T}^{\setminus i} = \boldsymbol{i}},
\end{equation}
for any non-zero event $\{Y_T = y\}$. Then, 
\begin{align*}
    \pr{Y_T=y\mid do(X^i_{t} = x,\boldsymbol{I}_{T}^{\setminus i} = \boldsymbol{i})} & = \pr{Y_T=y\mid X^i_{t} = x,\boldsymbol{I}_{T}^{\setminus i} = \boldsymbol{i}} & [\text{Eq. \ref{eq:2_new}}] \\
    & = \pr{Y_T=y\mid  \boldsymbol{I}_{T}^{\setminus i} = \boldsymbol{i}} & [\text{Eq. \ref{eq:granger}}] \\
    & = \pr{Y_T=y\mid X^i_{t} = x',\boldsymbol{I}_{T}^{\setminus i} = \boldsymbol{i}} & [\text{Eq. \ref{eq:granger}}] \\
    & = \pr{Y_T=y\mid do(X^i_{t} = x',\boldsymbol{I}_{T}^{\setminus i} = \boldsymbol{i})}. & [\text{Eq. \ref{eq:2_new}}]
\end{align*}
Hence, causality implies Granger causality.

We now prove that Granger causality implies causality. To this end, suppose that $X^i$ is not a potential cause of $Y$. By the definition of direct effects, it holds
\begin{align}
& \pr{Y_T=y\mid do(X^i_{t} = x,\boldsymbol{I}_{T}^{\setminus i} = \boldsymbol{i})} \nonumber \\
& \qquad \qquad \qquad = \expect{}{\pr{Y_T=y\mid do(X^i_{t} = x', \boldsymbol{I}_{T}^{\setminus i} = \boldsymbol{i})} \mid \boldsymbol{I}_{T}^{\setminus i} = \boldsymbol{i}}. \label{eq:lemma101}
\end{align}
Hence,
\begin{align*}
    & \pr{Y_T=y\mid do(X^i_{t} = x,\boldsymbol{I}_{T}^{\setminus i} = \boldsymbol{i})} & \\
    & \qquad  \qquad = \pr{Y_T=y\mid X^i_{t} = x,\boldsymbol{I}_{T}^{\setminus i} = \boldsymbol{i}} & [\text{Eq. \ref{eq:2_new}}]\\
    & \qquad  \qquad = \expect{}{\pr{Y_T=y\mid do(X^i_{t}, \boldsymbol{I}_{T}^{\setminus i})}\mid \boldsymbol{I}_{T}^{\setminus i} = \boldsymbol{i}} & [\text{Eq. \ref{eq:lemma101}}]\\
    & \qquad  \qquad = \expect{}{\pr{Y_T=y\mid X^i_{t} , \boldsymbol{I}_{T}^{\setminus i} }\mid \boldsymbol{I}_{T}^{\setminus i} = \boldsymbol{i}} & [\text{Eq. \ref{eq:2_new}}]\\
    & \qquad  \qquad = \pr{Y_T=y\mid  \boldsymbol{I}_{T}^{\setminus i} = \boldsymbol{i}}, & 
\end{align*}
and the claim follows. 
\end{proof}
\section{Proof of Theorem \ref{lemma:cond_variance}}
\label{app:cond_variance}
We prove the following result.
\cond*
In order to prove Theorem \ref{lemma:cond_variance}, we use the following auxiliary lemma.
\begin{lemma}
\label{lemma:expect}
Consider a causal model as in \ref{cond:1}-\ref{cond:3}. Then, the following conditions are equivalent:
\begin{enumerate}
    \item $\expect{}{Y_T \mid X^i_{t} = x,\boldsymbol{I}_{T}^{\setminus i} = \boldsymbol{i}} = \expect{}{Y_T \mid X^i_{t} = x', \boldsymbol{I}_{T}^{\setminus i} = \boldsymbol{i}}$  a.s. ; \label{itemone}
    \item $\pr{Y_T = y \mid X^i_{t} = x,\boldsymbol{I}_{T}^{\setminus i} = \boldsymbol{i}} = \pr{Y_T=y \mid X^i_{t} = x', \boldsymbol{I}_{T}^{\setminus i} = \boldsymbol{i}}$ a.s. \label{itemtwo}
\end{enumerate}
\end{lemma}
\begin{proof}
Clearly, Item \ref{itemtwo} implies Item \ref{itemone}. 

We now prove the converse, i.e., we show that Item \ref{itemone} implies Item \ref{itemtwo}. To this end,
define the group $\boldsymbol{P}_T $ consisting of all the causal parents of $Y_T$. Note that it holds $\boldsymbol{P}_T \subseteq \{\boldsymbol{I}_{T}^{\setminus i}, X^i_{t}\} \subseteq \{\boldsymbol{I}_{T}^{\setminus i}, X^i_{t}\}$. Hence, the joint intervention $\{ X^i_{t},\boldsymbol{i}^{i,t}_{T-1}\} \gets \{x,\boldsymbol{i}\}$ define an intervention on the parents $\boldsymbol{P}_T \gets \boldsymbol{p}$. Further, we can write the potential outcome as 
\begin{equation}
\label{eq:new_eq1}
Y_T\mid do(X^i_{t} = x,\boldsymbol{I}_{T}^{\setminus i} = \boldsymbol{i}) = f(\boldsymbol{p}) + \varepsilon . 
\end{equation}
Similarly, the joint intervention $\{X^i_{t},\boldsymbol{I}^{i,t}_{T-1}\} \gets \{x,\boldsymbol{i}\}$, define an intervention on the parents $\boldsymbol{P}_T \gets \boldsymbol{p}'$. We can write the potential outcome as 
\begin{equation}
\label{eq:new_eq2}
Y_T\mid do(X^i_{t} = x',\boldsymbol{I}_{T}^{\setminus i} = \boldsymbol{i}) = f(\boldsymbol{p}') + \varepsilon .
\end{equation}
Hence, it holds
\begin{align}
    f(\boldsymbol{p}) + \expect{}{\varepsilon} & = \expect{}{Y_T\mid do(X^i_{t} = x,\boldsymbol{I}_{T}^{\setminus i} = \boldsymbol{i})} & [\text{Eq. \ref{eq:new_eq1}}] \nonumber \\ 
    & = \expect{}{Y_T\mid X^i_{t} = x,\boldsymbol{I}_{T}^{\setminus i} = \boldsymbol{i}} & [\text{Eq. \ref{eq:2_new}, Theorem \ref{thm:granger}}] \nonumber \\ 
    & = \expect{}{Y_T\mid X^i_{t} = x',\boldsymbol{I}_{T}^{\setminus i} = \boldsymbol{i}} & [\text{by assumption}] \nonumber \\ 
    & = \expect{}{Y_T\mid do(X^i_{t} = x',\boldsymbol{I}_{T}^{\setminus i} = \boldsymbol{i})} & [\text{Eq. \ref{eq:2_new}, Theorem \ref{thm:granger}}] \nonumber \\ 
    & = f(\boldsymbol{p}')+ \expect{}{\varepsilon}.  & [\text{Eq. \ref{eq:new_eq2}}] \nonumber
\end{align}
By \ref{cond:1}, the variable $\varepsilon$ is exogenous independent noise. From the chain of equations above it follows that $f(\boldsymbol{p}) = f(\boldsymbol{p}')$. Hence, 
\begin{align}
& \pr{Y_T = y\mid do(X^i_{t} = x,\boldsymbol{I}_{T}^{\setminus i} = \boldsymbol{i})} = \pr{f(\boldsymbol{p}) + \varepsilon = y} \nonumber \\ 
& \qquad \qquad \qquad = \pr{f(\boldsymbol{p}') + \varepsilon = y} = \pr{Y_T = y\mid do(X^i_{t} = x',\boldsymbol{I}_{T}^{\setminus i} = \boldsymbol{i})} \label{eq:new_3}
\end{align}
We conclude that it holds
\begin{align*}
    & \pr{Y_T=y\mid X^i_{t} = x,\boldsymbol{I}_{T}^{\setminus i} = \boldsymbol{i}} & \\
    & \qquad = \pr{Y_T\mid do(X^i_{t} = x,\boldsymbol{I}_{T}^{\setminus i} = \boldsymbol{i})} & [\text{Eq. \ref{eq:2_new}, Theorem \ref{thm:granger}}] \\ 
    & \qquad = \pr{Y_T=y\mid do(X^i_{t} = x',\boldsymbol{I}_{T}^{\setminus i} = \boldsymbol{i})} & [\text{Eq. \ref{eq:new_3}}] \\ 
    & \qquad = \pr{Y_T=y\mid X^i_{t} = x',\boldsymbol{I}_{T}^{\setminus i} = \boldsymbol{i}}, & [\text{Eq. \ref{eq:2_new}, Theorem \ref{thm:granger}}]
\end{align*}
as claimed.
\end{proof}
We can now prove the main result.
\begin{proof}[Proof of Theorem \ref{lemma:cond_variance}]
We first prove that $X^i$ Granger causes $Y$ iff. it holds 
\begin{equation}
\label{eq:chi}
\expect{}{\left ( \expect{}{Y_T\mid X^i_{t} ,\boldsymbol{I}_{T}^{\setminus i} } - \expect{}{Y_T\mid \boldsymbol{I}_{T}^{\setminus i} } \right )^2} \neq 0.
\end{equation}

First, suppose that Eq. \ref{eq:chi} does not hold. Then, it holds $\mathbb{e}[Y_T \mid X^i_{t} = x,\boldsymbol{I}_{T}^{\setminus i} = \boldsymbol{i}] = \mathbb{E}[Y_T \mid X^i_{t} = x',\boldsymbol{I}_{T}^{\setminus i} = \boldsymbol{i}]$, a.s. Combining this equation with Lemma \ref{lemma:expect} yields
\begin{align*}
    \pr{Y_T = y \mid X^i_{t} = x,\boldsymbol{I}_{T}^{\setminus i} = \boldsymbol{i}} & = \expect{}{\pr{Y_T = y \mid X^i_{t} ,\boldsymbol{I}_{T}^{\setminus i} } \mid \boldsymbol{I}_{T}^{\setminus i} = \boldsymbol{i}}\\
    & = \pr{Y_T = y \mid \boldsymbol{I}_{T}^{\setminus i} = \boldsymbol{i}},
\end{align*}
a.s. Hence, if $X^i$ Granger causes $Y$, then Eq. \ref{eq:chi} holds.
\begin{align}
\label{eq:chi2}
    \expect{}{Y_T \mid X^i_{t} = x,\boldsymbol{I}_{T}^{\setminus i} = \boldsymbol{i}} \neq \expect{}{Y_T \mid X^i_{t} = x',\boldsymbol{I}_{T}^{\setminus i} = \boldsymbol{i}},
\end{align}
for a triple $\{x, x', \mathbf{w} \}$. By combining Eq. \ref{eq:chi2} with Lemma \ref{lemma:expect} we see that Eq. \ref{eq:chi} implies causality. However, by Theorem \ref{thm:granger} Granger causality is equivalent to causality in this case.

We now prove the claim. By the tower property of the expectation~\cite{Williams-1991} that
\begin{align*}
    & \expect{}{\left ( \expect{}{Y_T\mid X^i_{t},\boldsymbol{I}_{T}^{\setminus i}} - \expect{}{Y_T\mid \boldsymbol{I}_{T}^{\setminus i}} \right )^2} \\
    & =  \expect{}{\left (\expect{}{Y_T\mid X^i_{t},\boldsymbol{I}_{T}^{\setminus i}} - \expect{}{Y_T\mid \boldsymbol{I}_{T}^{\setminus i}} \right )^2} \\
    & = \expect{}{\expect{}{\left (\expect{}{Y_T\mid X^i_{t},\boldsymbol{I}_{T}^{\setminus i}} - \expect{}{Y_T\mid \boldsymbol{I}_{T}^{\setminus i}} \right )^2\mid \boldsymbol{I}_{T}^{\setminus i}}} \\
    & = \expect{}{\expect{}{\left (\expect{}{Y_T\mid X^i_{t},\boldsymbol{I}_{T}^{\setminus i}}^2 - \expect{}{Y_T\mid X^i_{t},\boldsymbol{I}_{T}^{\setminus i}}  \expect{}{Y_T\mid \boldsymbol{I}_{T}^{\setminus i}} \right )\mid \boldsymbol{I}_{T}^{\setminus i}}} \\
    & = \expect{}{\expect{}{Y_T\mid X^i_{t},\boldsymbol{I}_{T}^{\setminus i}}^2} - \expect{}{\expect{}{\left (\expect{}{Y_T\mid X^i_{t},\boldsymbol{I}_{T}^{\setminus i}}  \expect{}{Y_T\mid \boldsymbol{I}_{T}^{\setminus i}}\right )\mid \boldsymbol{I}_{T}^{\setminus i}}}\\
    & = \expect{}{\expect{}{Y_T\mid X^i_{t},\boldsymbol{I}_{T}^{\setminus i}}^2} - \expect{}{\expect{}{Y_T\mid \boldsymbol{I}_{T}^{\setminus i}}^2}\\
    & = \expect{}{Y_T\expect{}{Y_T\mid X^i_{t},\boldsymbol{I}_{T}^{\setminus i}}} - \expect{}{Y_T\expect{}{Y_T\mid \boldsymbol{I}_{T}^{\setminus i}}},
\end{align*}
as claimed.
\end{proof}

\section{Intuition on Zero-Masking}
\label{app:zero-mask}
We provide intuition why masking is a reasonably good idea. Assume that the function $\widehat{f}$ is a $\epsilon$-close estimator of the true function $f^*$ in the $\mathcal{L}^2(P_{X_1, X_2, \dots, X_m})$ norm, where the functions $\widehat{f}, f^*$ and the joint probability distribution $P_{X_2, \dots, X_m}$ are defined on the set $\{X_1, X_2, \dots, X_m\}$,

$$ \|\widehat{f} - f^*\|_{\mathcal{L}^2(P_{X_1, X_2, \dots, X_m})} \leq \epsilon$$

Now, let's mask the random variable $X_1$. We are interested to see how close is the estimator $\mathbb{E}_{X_1}\widehat{f}$ to the true function $\mathbb{E}_{X_1}f^*$ in the $\mathcal{L}^2(P_{X_2, \dots, X_m})$ norm, where the functions $\mathbb{E}_{X_1}\widehat{f}, \mathbb{E}_{X_1}f^*$ and the marginal probability distribution $P_{X_2, \dots, X_m}$ are defined on the rest of variables $\{X_2, \dots, X_m\}$, By Jensen's inequality, we infer that for any realization of $X_2=x_2, X_3=x_3, \dots, X_m=x_m$,
\begin{align*}
(\mathbb{E}_{X_1} \widehat{f}(X_1, x_2, \dots, x_m) & - \mathbb{E}_{X_1} f^*(X_1, x_2, \dots, x_m))^2 \\ & \leq \mathbb{E}_{X_1} [(\widehat{f}(X_1, x_2, \dots, x_m) - f^*(X_1, x_2, \dots, x_m))^2]
\end{align*}
Plugging it in the $\epsilon$-closeness assumption leads to,

$$ \|\mathbb{E}_{X_1}\widehat{f} - \mathbb{E}_{X_1}f^*\|_{\mathcal{L}^2(P_{X_2, \dots, X_m})} \leq \|\widehat{f} - f^*\|_{\mathcal{L}^2(P_{X_1, X_2, \dots, X_m})} \leq \epsilon, $$
which guarantees that $\mathbb{E}_{X_1}\widehat{f}$ is also $\epsilon$-closs to $\mathbb{E}_{X_1}f^*$ and hence it's a good estimator. In the sequel, a natural solution would be to estimate $\mathbb{E}_{X_1}\widehat{f}$ by taking averages of $\widehat{f}$ over different samples of $X_1$. However, for the linear regression problem that the estimator has a linear structure of the input, it is straightforward to show that it is enough to evaluate $\widehat{f}$ at $\mathbb{E} [X_1]$. And finally due to the zero-centering step of data preprocessing, $\mathbb{E} [X_1] = 0$. Thus, the aforementioned procedure is equivalent to zero-masking.

\section{A Note on the Number of Partitions} \label{app:k}
The number of partitions $k$ affects the performance of our algorithm in practice since a larger number of partitions will help in removing a bias in the estimates. However, in our experiments, we observe that a small number of partitions is sufficient to achieve good results. Furthermore, an excessive number of random partitions may have a detrimental effect, due to the possible small number of samples in each partition. Hence, we believe that the number of partitions will not drastically affect performance in practice. Reasonable choices of $k$ for our experiments range between 3-7, hence $k = \mathcal{O}(1)$ w.r.t. parameters of the problem. Thus, the resulting runtime can be reported as $\mathcal{O}(md)$.

\section{Extension to Full Causal Discovery} \label{sec:full}
It is possible to use Algorithm \ref{alg} for full causal discovery, for fully-observed acyclic auto-regressive models with no instantaneous effects (see \citet{DBLP:conf/nips/PetersJS13,DBLP:conf/clear2/LoweMSW22} for a precise definition of this restricted framework). In fact, under these more restrictive assumptions, we can identify the causes of each random variable of the model by testing Granger causality. For these models, we can learn the full summary graph, by identifying the causes of each variable with Algorithm \ref{alg}. The resulting run time can be quantified as $\mathcal{O}(mdk)$, where $d$ is the time complexity of performing a regression, as outlined above, $m$ is the number of time series considered, and $k$ is the number of dataset partitions used for double cross-fitting. Since $k = \mathcal{O}(1)$ w.r.t. parameters of the problem (see~\cref{app:k}), the runtime of DR-SIT is $\mathcal{O}(md)$.

\section{Computational Complexity and Comparison}
\label{app:comp}

As discussed in~\cref{app:related_work}, compared to DR-SIT, conditional independence-based approaches such as PCMCI~\citep{runge2019detecting}, PCMCI+~\citep{runge2020discovering}, and LPCMCI~\citep{gerhardus2020high} face exponential computational barriers. It is widely known that even endowed with a perfect infinite sample independence testing oracle, learning Bayesian Networks becomes NP-Hard~\citep{chickering2004large,chickering1996learning}. Consequently, computational challenges arise not only due to the nature of the conditional independence tests themselves but also from the computational intractability of searching through the exponentially large space of possible network structures. Hence, the runtime of the $\mathcal{O}(m)$ number of regressions that DR-SIT demands is negligible compared to the exponential number of conditional independence tests from lengthy time-series. To support this argument in practice, we provide a runtime comparison between DR-SIT and PCMCI+ w.r.t. the number of nodes $m$ in Table \ref{table:runtime_PCMCI}.

\begin{table}[t]
\centering
\caption{Table with runtime means and standard deviations for DR-SIT and PCMCI+ (in seconds).
  \label{table:runtime_PCMCI}}
\resizebox{\textwidth}{!}{
\begin{tabular}{lcccccccc}
      \toprule
       & 10 & 20 & 30 & 40 & 50 & 100 & 200 & 400\\
      \midrule
      DR-SIT  &  $42 \pm 13$ & $35 \pm 6$ & $29 \pm 1$ & $30 \pm 4$ & $35 \pm 15$ & $41 \pm 6$ & $62 \pm 6$& $77 \pm 10 $\\
      PCMCI+& $4.2 \pm 0.4$ & $18.6 \pm 0.8$ & $48.6 \pm 1.4$ & $99.2 \pm 10$ & $216.4 \pm 42.2$ & $1091 \pm 65$ & $5678 \pm 264 $ & $\approx 8$ hours\\
      \bottomrule
  \end{tabular}}
\end{table} 
%
\section{Additional Experiments}
\label{sec:appendix_experiments}

\subsection{Results for Synthetic Experiments}\label{sec:appendix_synthetic}

\begin{table}[ht]
  \centering
  \caption{Accuracy of our method for increasing number of potential causes $m$, and different noise-to-signal ration (\textsc{NSR}). We observe that our method maintains good accuracy, even in challenging settings with many potential causes and high noise.}
  \label{table:snthetic1}
{\scriptsize
\begin{tabular}{lccccccc}
\toprule
   & \multicolumn{7}{c}{\textbf{Accuracy}} \\
\cmidrule(lr){2-8}
$\boldsymbol{\feat}$   & $\boldsymbol{\nsr=0}$            &$\boldsymbol{\nsr=0.05}$        & $\boldsymbol{\nsr=0.1}$          &$\boldsymbol{\nsr=0.15}$        &$\boldsymbol{\nsr=0.2}$         &$\boldsymbol{\nsr=0.25}$        & $\boldsymbol{\nsr=0.3}$         \\
\hline
       $\boldsymbol{5}$ & 0.60 $\pm$ 0.09 & 0.74 $\pm$ 0.25 & 0.66 $\pm$ 0.19 & 0.90 $\pm$ 0.11 & 0.80 $\pm$ 0.06 & 0.82 $\pm$ 0.10 & 0.94 $\pm$ 0.12 \\
       $\boldsymbol{10}$& 0.99 $\pm$ 0.02 & 0.92 $\pm$ 0.08 & 0.97 $\pm$ 0.04 & 0.79 $\pm$ 0.18 & 0.89 $\pm$ 0.06 & 0.94 $\pm$ 0.04 & 0.90 $\pm$ 0.08 \\
       $\boldsymbol{15}$& 0.89 $\pm$ 0.05 & 0.89 $\pm$ 0.05 & 0.88 $\pm$ 0.10 & 0.85 $\pm$ 0.02 & 0.79 $\pm$ 0.11 & 0.79 $\pm$ 0.09 & 0.81 $\pm$ 0.03 \\
       $\boldsymbol{20}$& 0.83 $\pm$ 0.07 & 0.73 $\pm$ 0.05 & 0.72 $\pm$ 0.07 & 0.77 $\pm$ 0.10 & 0.73 $\pm$ 0.05 & 0.75 $\pm$ 0.05 & 0.69 $\pm$ 0.06 \\
       $\boldsymbol{25}$& 0.76 $\pm$ 0.04 & 0.72 $\pm$ 0.10 & 0.71 $\pm$ 0.06 & 0.63 $\pm$ 0.08 & 0.68 $\pm$ 0.04 & 0.71 $\pm$ 0.07 & 0.66 $\pm$ 0.04 \\
       $\boldsymbol{30}$ & 0.76 $\pm$ 0.04 & 0.74 $\pm$ 0.04 & 0.72 $\pm$ 0.07 & 0.70 $\pm$ 0.10 & 0.66 $\pm$ 0.05 & 0.68 $\pm$ 0.04 & 0.64 $\pm$ 0.07 \\
       $\boldsymbol{35}$& 0.68 $\pm$ 0.02 & 0.65 $\pm$ 0.07 & 0.72 $\pm$ 0.06 & 0.65 $\pm$ 0.03 & 0.66 $\pm$ 0.05 & 0.61 $\pm$ 0.07 & 0.64 $\pm$ 0.09 \\
       $\boldsymbol{40}$ & 0.64 $\pm$ 0.03 & 0.69 $\pm$ 0.05 & 0.67 $\pm$ 0.05 & 0.62 $\pm$ 0.03 & 0.63 $\pm$ 0.06 & 0.58 $\pm$ 0.07 & 0.60 $\pm$ 0.07 \\
       $\boldsymbol{45}$ & 0.65 $\pm$ 0.04 & 0.68 $\pm$ 0.08 & 0.58 $\pm$ 0.04 & 0.61 $\pm$ 0.03 & 0.64 $\pm$ 0.04 & 0.59 $\pm$ 0.04 & 0.61 $\pm$ 0.06 \\
       $\boldsymbol{50}$ & 0.68 $\pm$ 0.05 & 0.63 $\pm$ 0.05 & 0.64 $\pm$ 0.06 & 0.63 $\pm$ 0.05 & 0.66 $\pm$ 0.05 & 0.59 $\pm$ 0.08 & 0.64 $\pm$ 0.07 \\
\bottomrule
\end{tabular}
}
\end{table}

\begin{table}[ht]
  \centering
  \caption{CSI Score of Algorithm \ref{alg} for increasing number of potential causes $m$, and different noise-to-signal ration (\textsc{NSR}). Again, we observe that our method is robust to increasing \textsc{NSR}.}
  \label{table:snthetic3}
{\scriptsize
\begin{tabular}{lccccccc}
\toprule
   & \multicolumn{7}{c}{\textbf{CSI Score}} \\
\cmidrule(lr){2-8}
$\boldsymbol{\feat}$    & $\boldsymbol{\nsr=0}$            &$\boldsymbol{\nsr=0.05}$        & $\boldsymbol{\nsr=0.1}$          &$\boldsymbol{\nsr=0.15}$        &$\boldsymbol{\nsr=0.2}$         &$\boldsymbol{\nsr=0.25}$        & $\boldsymbol{\nsr=0.3}$         \\
\hline
      $\boldsymbol{5}$ & 0.57 $\pm$ 0.07 & 0.71 $\pm$ 0.28 & 0.63 $\pm$ 0.19 & 0.86 $\pm$ 0.12 & 0.69 $\pm$ 0.10 & 0.73 $\pm$ 0.12 & 0.91 $\pm$ 0.17 \\
    $\boldsymbol{10}$ & 0.98 $\pm$ 0.04 & 0.86 $\pm$ 0.14 & 0.95 $\pm$ 0.06 & 0.69 $\pm$ 0.21 & 0.80 $\pm$ 0.09 & 0.87 $\pm$ 0.07 & 0.82 $\pm$ 0.14 \\
    $\boldsymbol{15}$ & 0.80 $\pm$ 0.07 & 0.77 $\pm$ 0.11 & 0.75 $\pm$ 0.16 & 0.68 $\pm$ 0.06 & 0.59 $\pm$ 0.19 & 0.57 $\pm$ 0.15 & 0.59 $\pm$ 0.03 \\
     $\boldsymbol{20}$ & 0.66 $\pm$ 0.12 & 0.52 $\pm$ 0.12 & 0.46 $\pm$ 0.07 & 0.56 $\pm$ 0.10 & 0.46 $\pm$ 0.08 & 0.51 $\pm$ 0.07 & 0.39 $\pm$ 0.09 \\
     $\boldsymbol{25}$ & 0.51 $\pm$ 0.06 & 0.45 $\pm$ 0.12 & 0.43 $\pm$ 0.09 & 0.37 $\pm$ 0.09 & 0.38 $\pm$ 0.05 & 0.41 $\pm$ 0.08 & 0.33 $\pm$ 0.06 \\
     $\boldsymbol{30}$  & 0.47 $\pm$ 0.05 & 0.50 $\pm$ 0.07 & 0.46 $\pm$ 0.12 & 0.38 $\pm$ 0.08 & 0.35 $\pm$ 0.08 & 0.34 $\pm$ 0.05 & 0.31 $\pm$ 0.07 \\
     $\boldsymbol{35}$  & 0.42 $\pm$ 0.04 & 0.31 $\pm$ 0.06 & 0.39 $\pm$ 0.07 & 0.29 $\pm$ 0.07 & 0.30 $\pm$ 0.09 & 0.22 $\pm$ 0.09 & 0.26 $\pm$ 0.09 \\
     $\boldsymbol{40}$  & 0.32 $\pm$ 0.06 & 0.38 $\pm$ 0.05 & 0.33 $\pm$ 0.08 & 0.29 $\pm$ 0.06 & 0.24 $\pm$ 0.06 & 0.20 $\pm$ 0.07 & 0.19 $\pm$ 0.11 \\
     $\boldsymbol{45}$  & 0.33 $\pm$ 0.10 & 0.34 $\pm$ 0.07 & 0.20 $\pm$ 0.02 & 0.22 $\pm$ 0.05 & 0.25 $\pm$ 0.03 & 0.20 $\pm$ 0.05 & 0.19 $\pm$ 0.06 \\
     $\boldsymbol{50}$ & 0.33 $\pm$ 0.03 & 0.29 $\pm$ 0.06 & 0.29 $\pm$ 0.06 & 0.23 $\pm$ 0.04 & 0.26 $\pm$ 0.06 & 0.21 $\pm$ 0.08 & 0.26 $\pm$ 0.07 \\
\bottomrule
\end{tabular}
}
\end{table}
\begin{table}[ht]
  \centering
  \caption{F1 Score of the \alg for increasing number of potential causes $m$, and different noise-to-signal ration (\textsc{NSR}). Interestingly, our method maintains a good F1 score for increasing \textsc{NSR}.}
  \label{table:snthetic2}
{\scriptsize
\begin{tabular}{lccccccc}
\toprule
   & \multicolumn{7}{c}{\textbf{F1 Score}} \\
\cmidrule(lr){2-8}
$\boldsymbol{\feat}$ & $\boldsymbol{\nsr=0}$            &$\boldsymbol{\nsr=0.05}$        & $\boldsymbol{\nsr=0.1}$          &$\boldsymbol{\nsr=0.15}$        &$\boldsymbol{\nsr=0.2}$         &$\boldsymbol{\nsr=0.25}$        & $\boldsymbol{\nsr=0.3}$         \\
\hline
       $\boldsymbol{5}$ & 0.73 $\pm$ 0.06 & 0.79 $\pm$ 0.20 & 0.75 $\pm$ 0.12 & 0.92 $\pm$ 0.07 & 0.81 $\pm$ 0.07 & 0.84 $\pm$ 0.08 & 0.95 $\pm$ 0.11 \\
       $\boldsymbol{10}$ & 0.99 $\pm$ 0.02 & 0.92 $\pm$ 0.09 & 0.98 $\pm$ 0.03 & 0.80 $\pm$ 0.15 & 0.89 $\pm$ 0.06 & 0.93 $\pm$ 0.04 & 0.90 $\pm$ 0.09 \\
       $\boldsymbol{15}$ & 0.89 $\pm$ 0.05 & 0.86 $\pm$ 0.08 & 0.85 $\pm$ 0.11 & 0.81 $\pm$ 0.04 & 0.73 $\pm$ 0.15 & 0.72 $\pm$ 0.12 & 0.74 $\pm$ 0.03 \\
       $\boldsymbol{20}$ & 0.79 $\pm$ 0.08 & 0.67 $\pm$ 0.10 & 0.62 $\pm$ 0.07 & 0.71 $\pm$ 0.09 & 0.63 $\pm$ 0.07 & 0.67 $\pm$ 0.06 & 0.56 $\pm$ 0.09 \\
       $\boldsymbol{25}$ & 0.68 $\pm$ 0.05 & 0.61 $\pm$ 0.11 & 0.60 $\pm$ 0.09 & 0.53 $\pm$ 0.11 & 0.55 $\pm$ 0.05 & 0.57 $\pm$ 0.08 & 0.49 $\pm$ 0.07 \\
       $\boldsymbol{30}$ & 0.64 $\pm$ 0.05 & 0.66 $\pm$ 0.06 & 0.62 $\pm$ 0.10 & 0.54 $\pm$ 0.09 & 0.51 $\pm$ 0.09 & 0.50 $\pm$ 0.06 & 0.46 $\pm$ 0.09 \\
       $\boldsymbol{35}$ & 0.59 $\pm$ 0.05 & 0.47 $\pm$ 0.07 & 0.56 $\pm$ 0.07 & 0.45 $\pm$ 0.08 & 0.46 $\pm$ 0.11 & 0.35 $\pm$ 0.11 & 0.41 $\pm$ 0.11 \\
       $\boldsymbol{40}$ & 0.49 $\pm$ 0.07 & 0.55 $\pm$ 0.05 & 0.49 $\pm$ 0.08 & 0.45 $\pm$ 0.07 & 0.38 $\pm$ 0.08 & 0.33 $\pm$ 0.09 & 0.30 $\pm$ 0.15 \\
       $\boldsymbol{45}$ & 0.49 $\pm$ 0.11 & 0.50 $\pm$ 0.08 & 0.34 $\pm$ 0.03 & 0.36 $\pm$ 0.06 & 0.40 $\pm$ 0.04 & 0.33 $\pm$ 0.06 & 0.32 $\pm$ 0.08 \\
       $\boldsymbol{50}$ & 0.50 $\pm$ 0.03 & 0.44 $\pm$ 0.07 & 0.44 $\pm$ 0.08 & 0.38 $\pm$ 0.05 & 0.41 $\pm$ 0.07 & 0.34 $\pm$ 0.11 & 0.40 $\pm$ 0.09 \\
\bottomrule
\end{tabular}
}
\end{table}

\begin{figure}[t]
    \centering
    \includegraphics[width=6in]
    {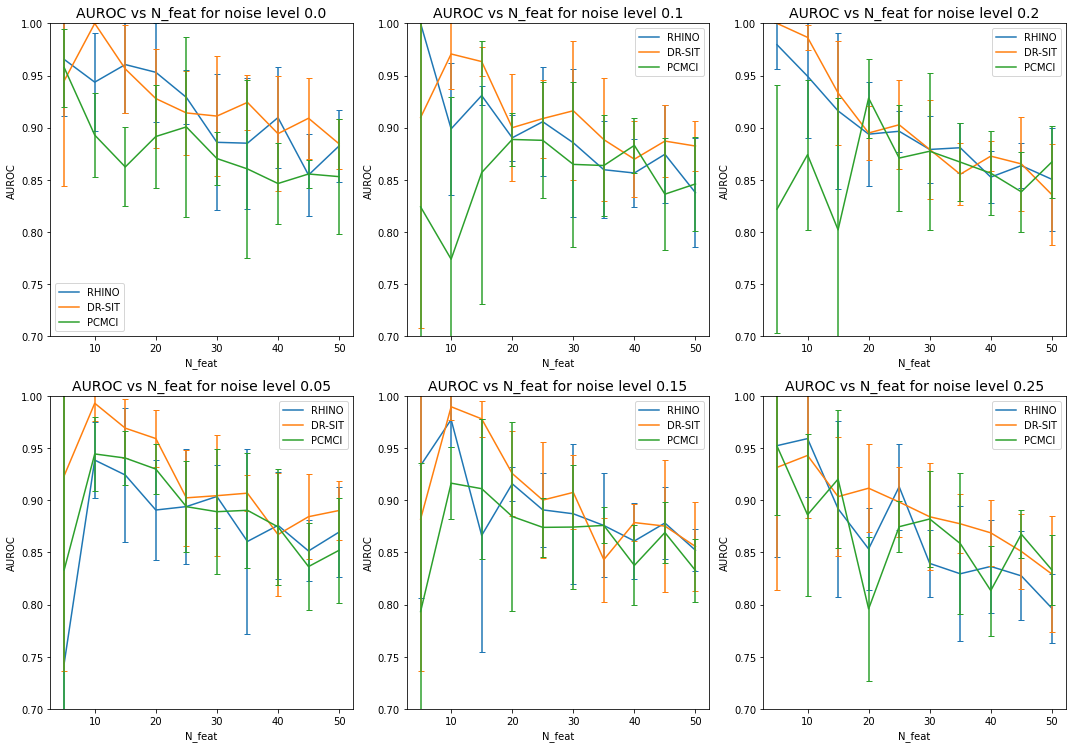}
    \caption{AUROC metric for DR-SIT, RHINO and PCMCI for various noise levels on the synthetic dataset.}
    \label{AUROC_synthetic_dataset}
\end{figure}

\subsection{Performance in Low-Sample Regimes}
\label{sec:low_sample}
 The double robustness property enables our algorithm to rely on simple estimators with low statistical complexity. As a result, our method shows more consistent performance in low-sample regimes as opposed to existing approaches that are based on overparameterized models demanding so many data points (\Cref{fig:exp_observations}).

\begin{figure*}[h!]
   \centering
   \begin{subfigure}[t]{2.7in}
       \centering
       \includegraphics[width=2.7in]{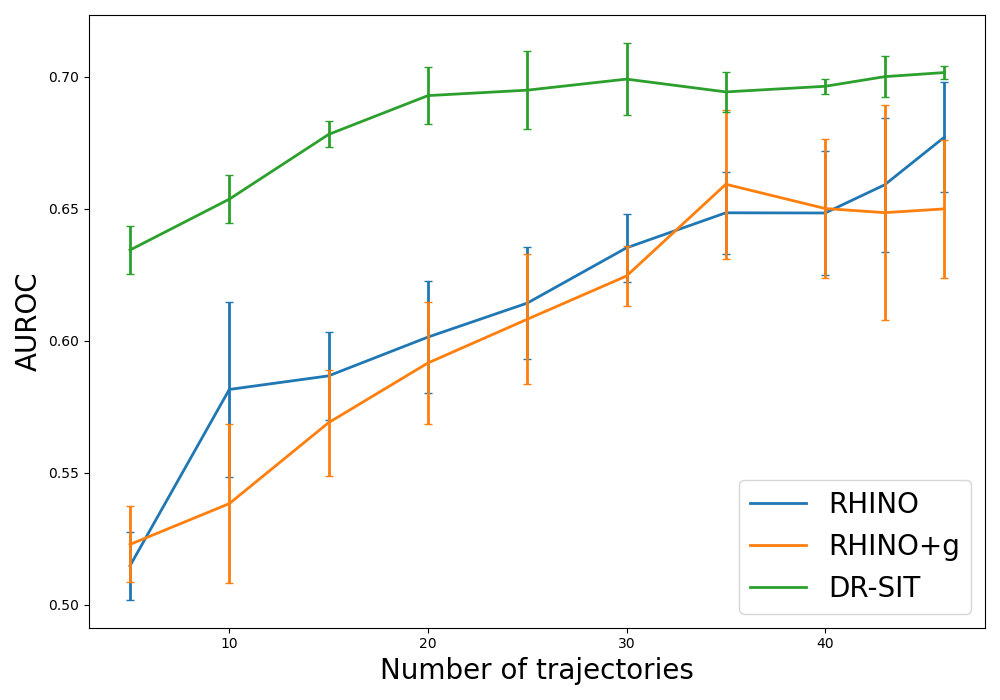}
       \caption{Task: E.Coli 1}\label{fig:1a}        
   \end{subfigure}
   \begin{subfigure}[t]{2.7in}
       \centering
       \includegraphics[width=2.7in]{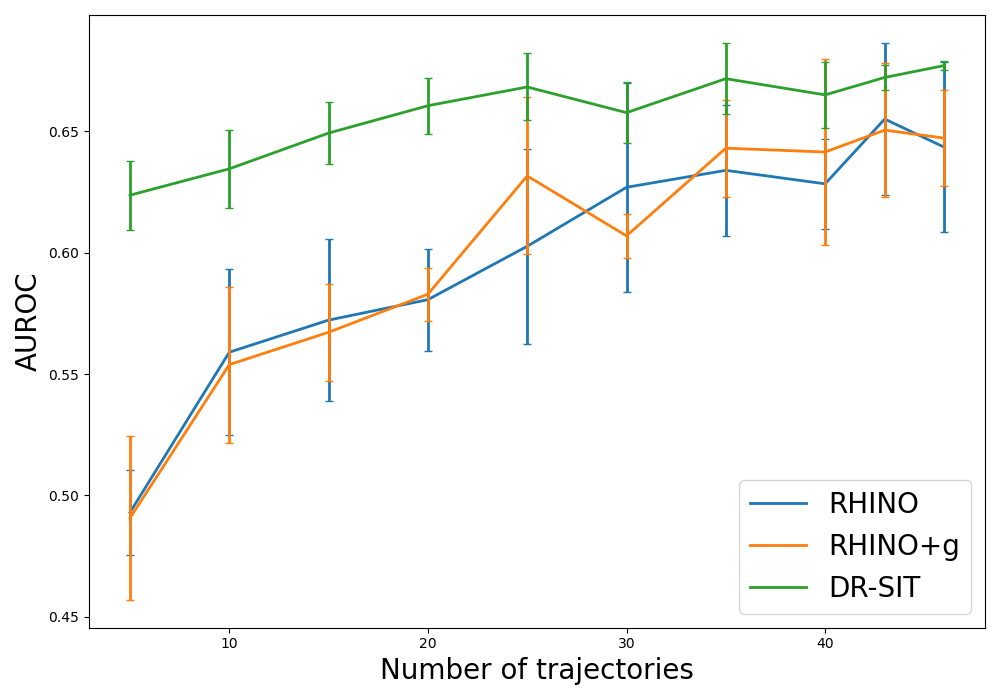}
       \caption{Task: E.Coli 2}\label{fig:1b}    
  \end{subfigure}
  \begin{subfigure}[t]{2.7in}
       \centering
       \includegraphics[width=2.7in]{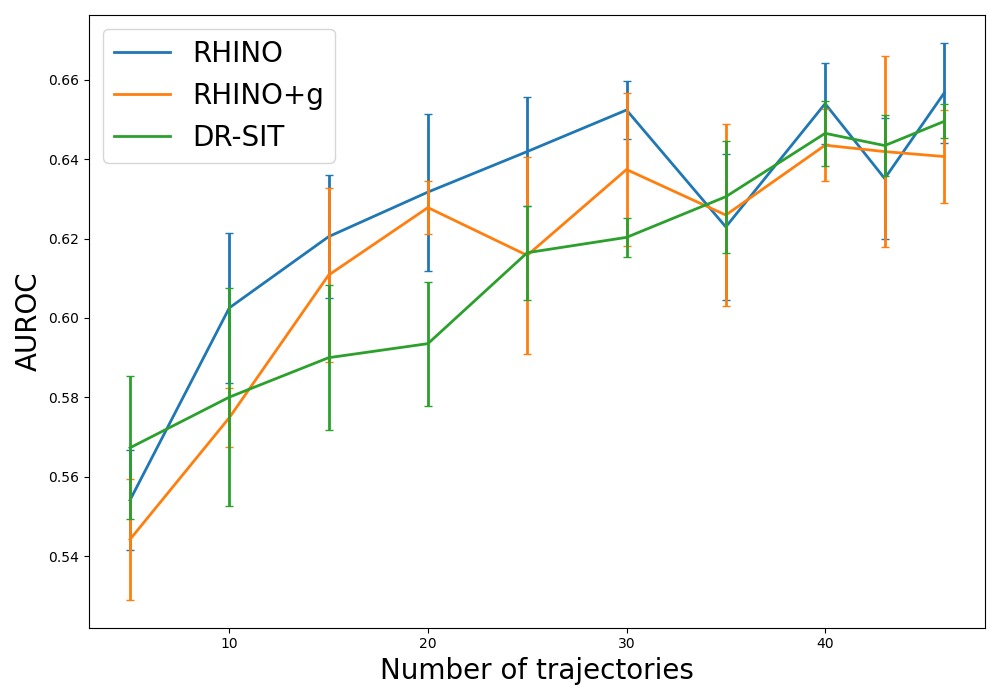}
       \caption{Task: Yeast 1}\label{fig:1c}    
  \end{subfigure}
  \begin{subfigure}[t]{2.7in}
       \centering
       \includegraphics[width=2.7in]{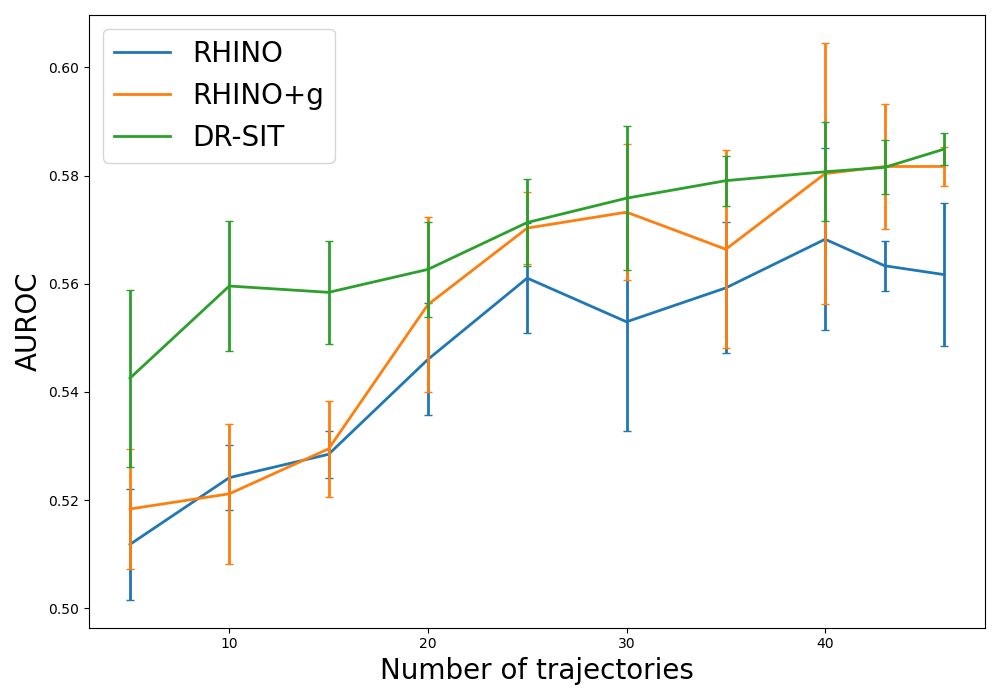}
       \caption{Task: Yeast 2}\label{fig:1d}    
   \end{subfigure}
   \begin{subfigure}[t]{2.7in}
       \centering
       \includegraphics[width=2.7in]{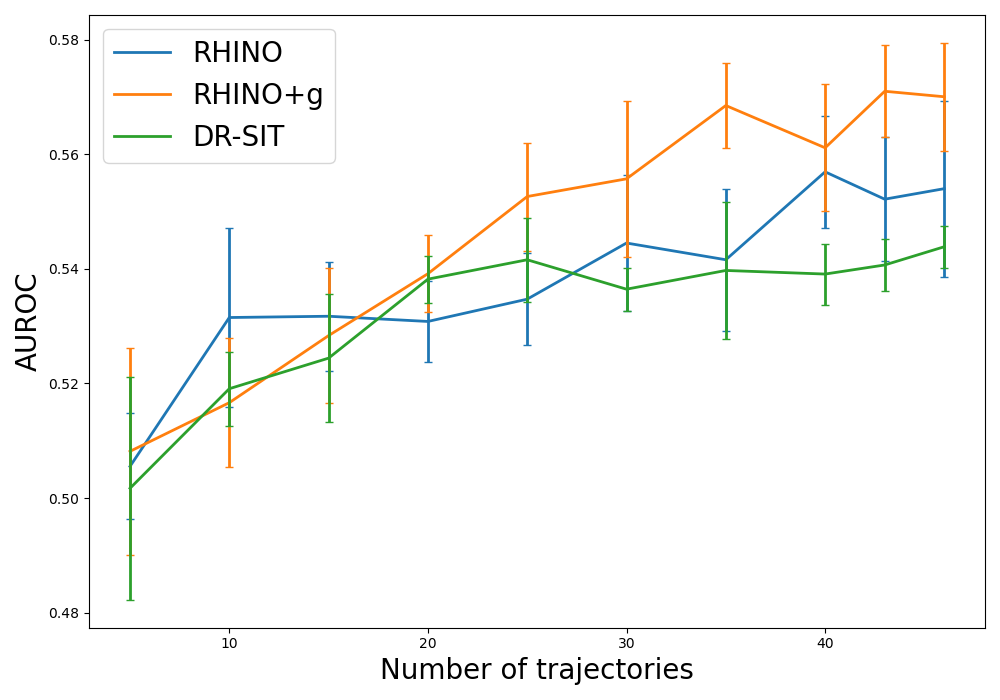}
       \caption{Task: Yeast 3}\label{fig:1e}    
   \end{subfigure}
   \caption{This figure demonstrates the consistent performance of DR-SIT w.r.t number of observations compared to state-of-the-art methods Rhino and Rhino+g. Note that Rhino and Rhino+g are built on neural networks. DR-SIT significantly outperforms Rhino and Rhino+g in E.Coli 1 and E.Coli 2 and shows competitive results in Yeast 1. 
   Thanks to the double robustness property of DR-SIT, the dependence of our algorithm on the estimator is much lower than the well-established approaches. In this regard, DR-SIT with a simple kernel regression with polynomial kernels has superior performance compared to state-of-the-art methods Rhino and Rhino+g. This superiority gets magnified in the low number of observation regimes due to the high sample complexity required by Rhino and Rhino+g.}\label{fig:exp_observations}
\end{figure*}

Moreover, in \cref{fig:auroc_vs_time_Ecoli1,fig:auroc_vs_time_Ecoli2,fig:auroc_vs_time_Yeast1,fig:auroc_vs_time_Yeast2,fig:auroc_vs_time_Yeast3} we compare the progression of AUROC score vs the total runtime for DR-SIT vs RHINO in various combinations of tasks and training set sizes. The hardware specifications are described in \cref{sec:runtime} and the training hyperparameter settings for RHINO in \cref{sec:exp_semi}. The runtime of DR-SIT is always less than 1 minute (so afterwards AUROC curve is plotted as a constant) where each epoch for RHINO takes about 30 seconds independently of the training dataset size.


\begin{figure}
    \centering
    \includegraphics[height=0.9\textheight]{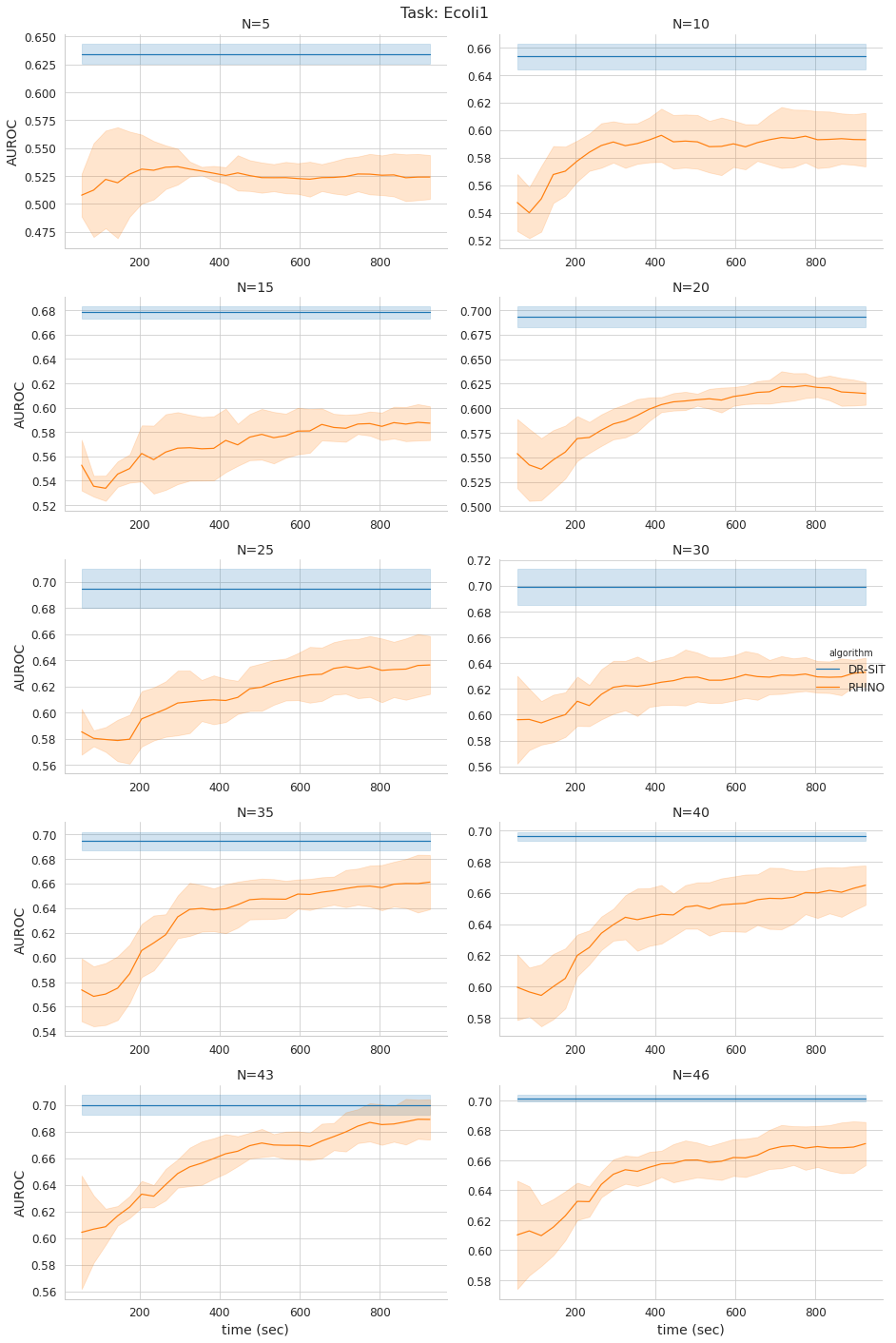}
    \caption{AUROC vs time (in secs) for DR-SIT (blue) vs RHINO (orange) for Ecoli 1 task and various numbers of training observations (number of trajectories). }
    \label{fig:auroc_vs_time_Ecoli1}
\end{figure}

\begin{figure}
    \centering
    \includegraphics[height=0.9\textheight]{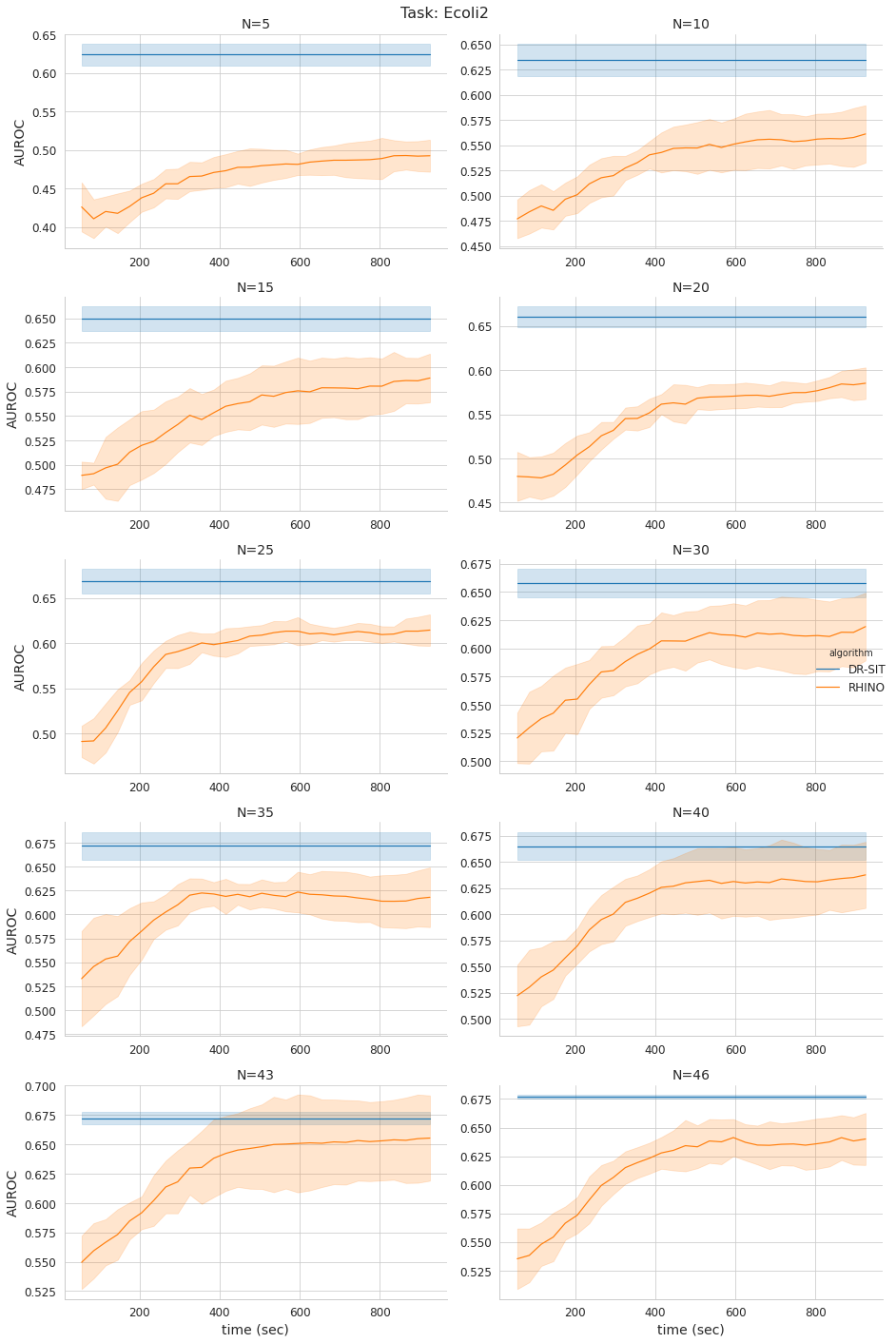}
    \caption{AUROC vs time (in secs) for DR-SIT (blue) vs RHINO (orange) for Ecoli 2 task and various numbers of training observations (number of trajectories). }
    \label{fig:auroc_vs_time_Ecoli2}
\end{figure}

\begin{figure}
    \centering
    \includegraphics[height=0.9\textheight]{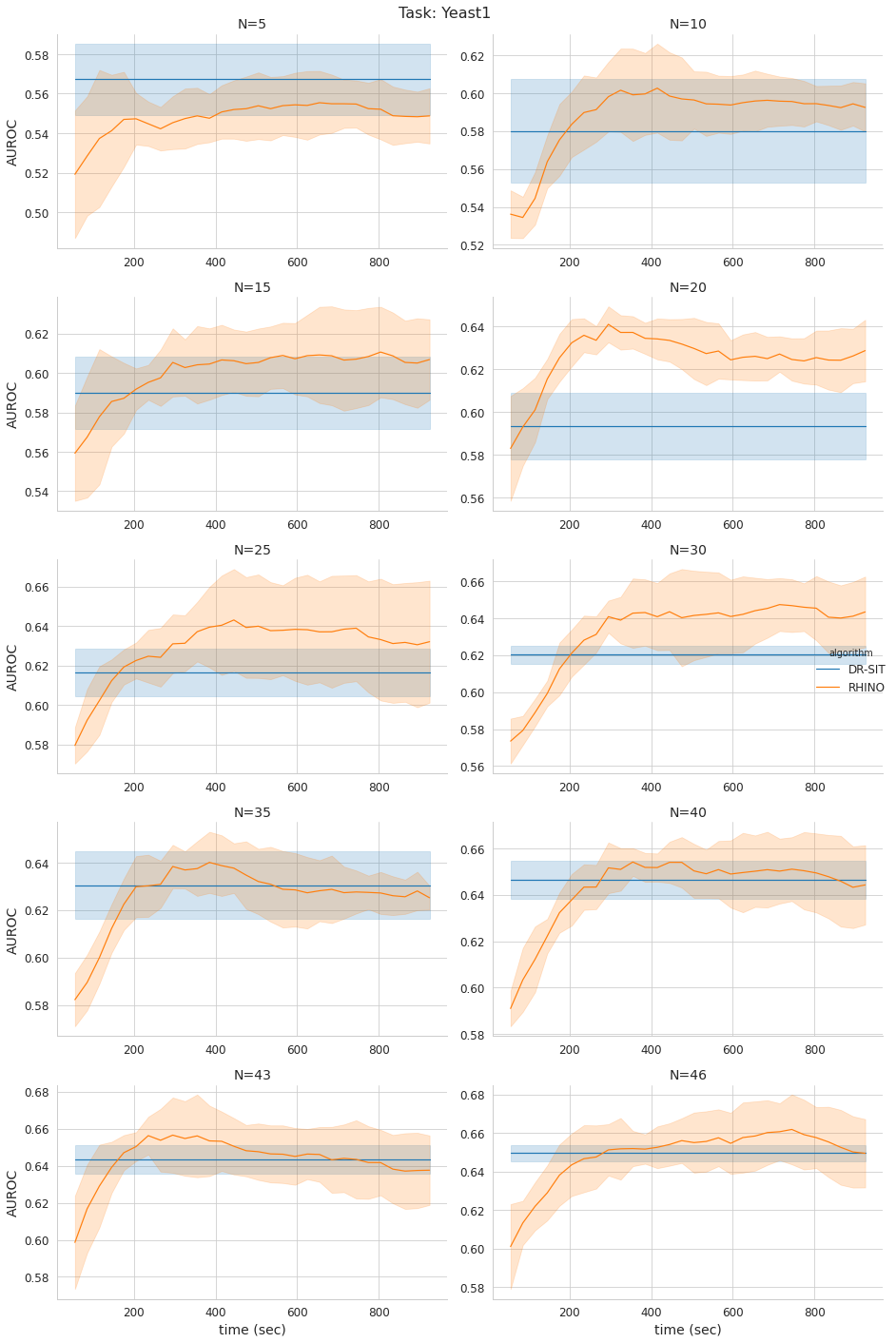}
    \caption{AUROC vs time (in secs) for DR-SIT (blue) vs RHINO (orange) for Yeast 1 task and various numbers of training observations (number of trajectories). }
    \label{fig:auroc_vs_time_Yeast1}
\end{figure}

\begin{figure}
    \centering
    \includegraphics[height=0.9\textheight]{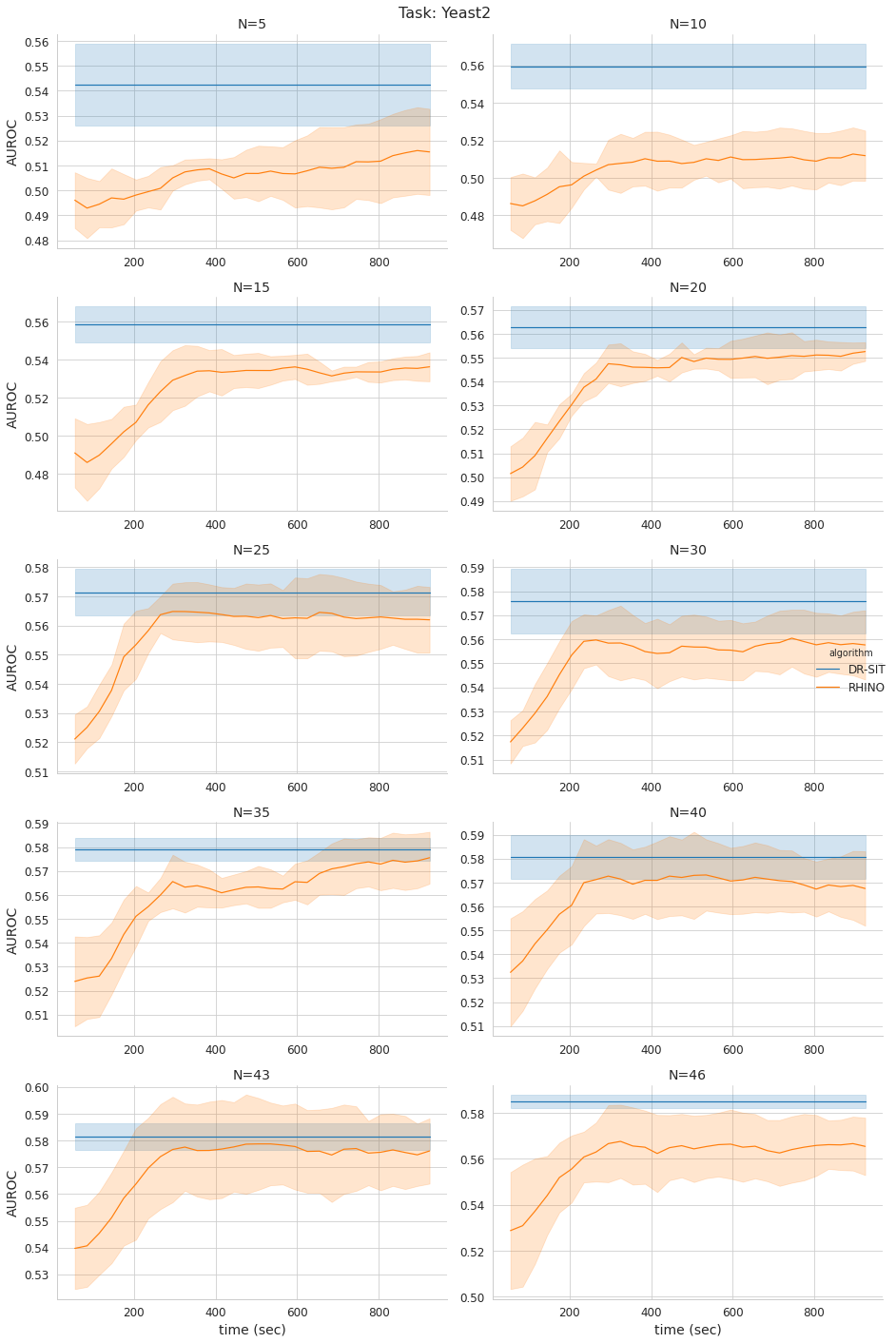}
    \caption{AUROC vs time (in secs) for DR-SIT (blue) vs RHINO (orange) for Yeast 2 task and various numbers of training observations (number of trajectories). }
    \label{fig:auroc_vs_time_Yeast2}
\end{figure}

\begin{figure}
    \centering
    \includegraphics[height=0.9\textheight]{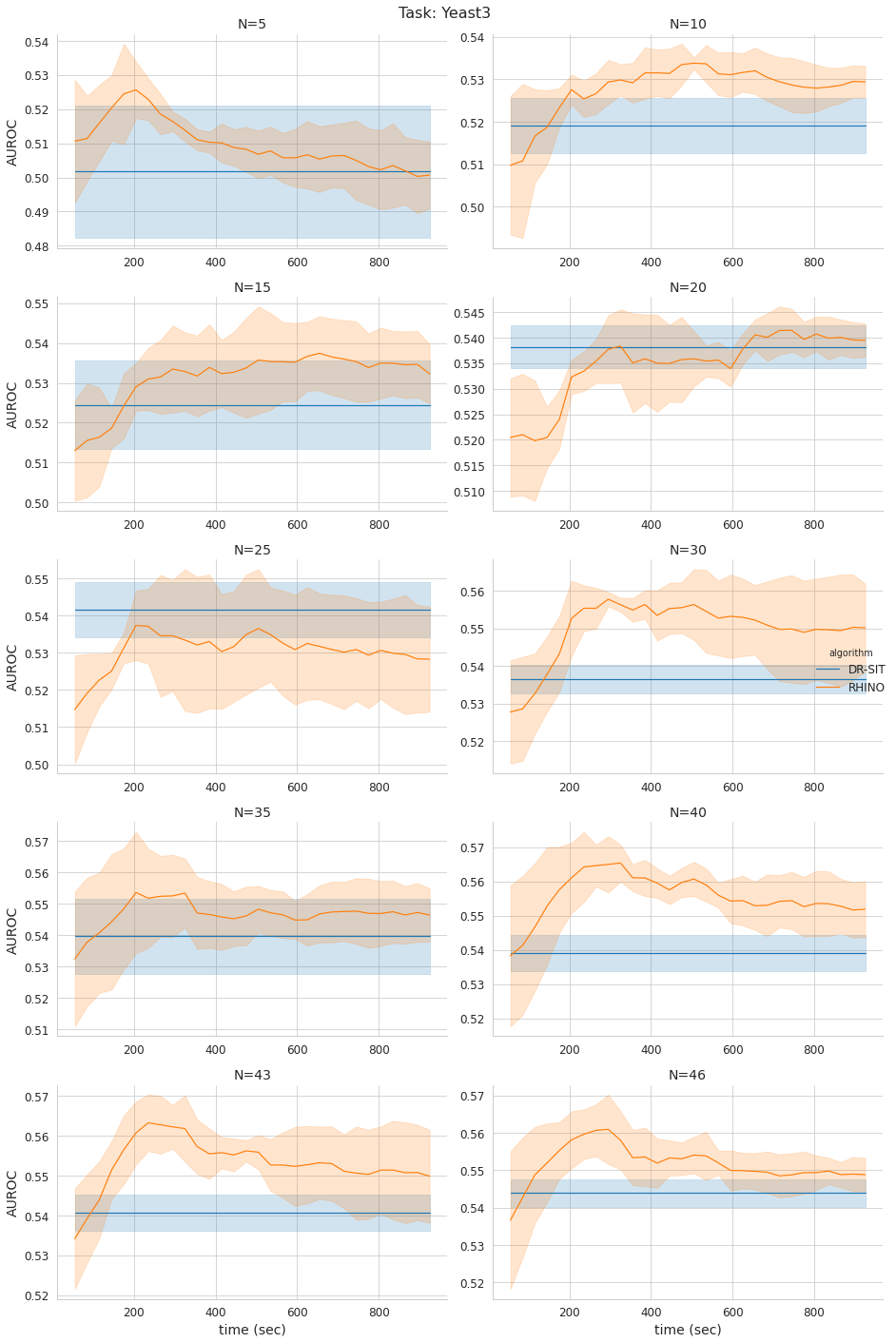}
    \caption{AUROC vs time (in secs) for DR-SIT (blue) vs RHINO (orange) for Yeast 3 task and various numbers of training observations (number of trajectories). }
    \label{fig:auroc_vs_time_Yeast3}
\end{figure}

\newpage

\vfill

\end{document}